\documentclass[11pt]{article}

\usepackage{mathrsfs,amsmath,amsfonts,amssymb,amstext,amscd,bm,bbm,dsfont,pifont, 
            amsthm,stmaryrd,euscript,color,xcolor,accents,xr} 
\usepackage[backref]{hyperref}
\usepackage{algorithmic}
\usepackage{algorithm}
\usepackage{url}
\usepackage{mathtools}
\usepackage{epsfig}
\usepackage{float}
\usepackage{appendix}
\usepackage{amstext}
\usepackage{floatflt}
\usepackage{nicefrac}
\usepackage{amsmath}
\usepackage{anysize,hyperref}
\usepackage{enumerate}
\usepackage{epsfig}
\usepackage{hyperref}
\usepackage{graphicx}
\usepackage{mathtools}
\usepackage{ upgreek }
\usepackage{pdfpages}
\numberwithin{equation}{section} 
\input{macros_nicole}
\newcommand{\heavy}{\theta} 
\usepackage[linecolor=magenta!60!, backgroundcolor=magenta!10!,textwidth=1.6cm, textcolor=magenta]{todonotes}

\title{Random feature approximation for general spectral methods}

\author{Mike Nguyen\footnote{Corresponding Author} \\
Technical University  of Braunschweig \\
\texttt{mike.nguyen@tu-braunschweig.de} 
\and
Nicole M\"ucke\\
Technical University  of Braunschweig  \\ 
\texttt{nicole.muecke@tu-braunschweig.de}
}
\date{\today}

\begin{document}

\maketitle

\begin{abstract}
  Random feature approximation is arguably one of the most popular techniques to speed up kernel methods in large scale algorithms and provides a theoretical approach to the analysis of deep neural networks. We analyze generalization properties for a large class of spectral regularization methods combined with random features, containing kernel methods with implicit regularization such as gradient descent or explicit methods like Tikhonov regularization. For our estimators we obtain optimal learning rates over regularity classes (even for classes that are not included in the reproducing kernel Hilbert space), which are defined through appropriate source conditions. This improves or completes previous results obtained in related settings for specific kernel algorithms.
\end{abstract}


\section{Introduction}

The rapid technological progress has led to accumulation of vast amounts of high-dimensional data in recent years. Consequently, to analyse such amounts of data it is no longer sufficient to create algorithms that solely aim for the best possible predictive accuracy. Instead, there is a pressing need to design algorithms that can efficiently process large datasets while minimizing computational overhead. In light of these challenges, two fundamental algorithmic tools, fast gradient methods, and sketching techniques, have emerged. Iterative gradient methods such as acceleration methods \cite{pagliana2019implicit} or stochastic gradient methods \cite{SGDfeatures} leading to favorable convergence rates while reducing computational complexity during learning.  On the other hand sketching techniques enable the reduction of data dimension, thereby decreasing memory requirements through random projections. The allure of combining both methodologies has garnered significant attention from researchers and practitioners alike. 
Especially for Kernel based algorithms various sketching tools have gained a lot of attention in recent years. For nonparametric statistical approaches kernel methods are in many applications still state of the 
art and provide an elegant and effective framework to develop theoretical optimal learning bounds \cite{Muecke2017op.rates, Lin_2020, spectral.rates}. However those benefits come with a computational cost making these methods unfeasible when dealing with large datasets. In fact traditional kernelized learning algorithms require storing the kernel gram matrix $\mathbf{K}_{i,j}=K(x_i,x_j)$ where $K(.,.)$ denotes the kernel function and $x_i,x_j$ the data points. This results in a memory cost of at least $O(n^2)$ and a time cost of up to $O(n^3)$ where $n$ denotes the data set size \cite{kernellearning}. Most popular sketching tools to overcome these issues are Nyström approximations \cite{rudi2016more} and random feature approximation (RFA) \cite{pmlrv119zhen20a, features}. 
In this paper, we investigate algorithms, using the interplay of fast learning methods and RFA and analyse generalization performance of such algorithms. Related work was contributed by \cite{features} and \cite{SGDfeatures}. They obtained optimal rates for Kernel Ridge Regression (KRR) and Stochastic Gradient Descent respectively, both algorithms were combined with RFA. Using a general spectral filtering framework \cite{Caponetto} we proved fast rates for all kind of learning methods with implicit or explicit regularization. For example gradient descent, acceleration methods and we also cover the results of \cite{features} for KRR. Moreover, we managed to overcome the saturation effect appearing in \cite{features} and \cite{SGDfeatures} by providing fast rates of convergence for objectives with any degree of smoothness.
The rest of the paper is organized as follows. In Section 2, we present our setting and review relevant results
on learning with kernels, and learning with random features. In Section 3, we
present and discuss our main results, while proofs are deferred to the appendix. Finally,
numerical experiments are presented in Section 4.


{\bf Notation.} 
By $\cL(\cH_1, \cH_2)$ we denote the space of bounded linear operators between real Hilbert spaces $\cH_1$, $\cH_2$. 
We write $\cL(\cH, \cH) = \cL(\cH)$. For $\Gamma \in \cL(\cH)$ we denote by $\Gamma^T$ the adjoint operator and for compact $\Gamma$ 
by $(\lam_j(\Gamma))_{j}$ the  sequence of eigenvalues. If $\theta \in \cH$ we write $\theta \otimes \theta := \langle \cdot, \theta \ra \theta$. 
We let $[n]=\{1,...,n\}$.  For two positive sequences $(a_n)_n$, $(b_n)_n$ we write $a_n \lesssim b_n$ if $a_n \leq c b_n$
for some $c>0$ and $a_n \simeq b_n$ if both $a_n \lesssim b_n$ and $b_n \lesssim a_n$.


\section{Setup}
\label{sec:setting}

We let $\cX \subset \mbr^d$ be the input space and $\cY\subset \mbr$ be the output space. 
The unknown data distribution on the data space $\cZ=\cX \times \cY$ is denoted by $\rho$ while the marginal distribution on $\cX$ is denoted as 
$\rho_X$ and the regular conditional distribution on $\cY$ given $x \in \cX$ is denoted by $\rho(\cdot | x)$, see e.g. \cite{Shao_2003_book}. 

Given a measurable function $g: \cX \to \mbr$ we further define the expected risk as 
\begin{equation}
\label{eq:expected-risk}
\cE(g) := \mbe[ \ell (g(X), Y) ]\;,
\end{equation}  
where the expectation is taken w.r.t. the distribution $\rho$ and $\ell: \mbr \times \cY \to \mbr_+$ is the least-square loss 
$\ell(t, y)=\frac{1}{2}(t-y)^2$. It is known that the global minimizer  of $\cE$ over the set of all measurable functions is 
given by the regression function $g_\rho(x)= \int_{\cY} y \rho(dy|x)$.

\subsection{Motivation of Kernel Methods with RFA }

Kernel methods are nonparametric approaches defined by a kernel $K:\mathcal{X}\times\mathcal{X} \rightarrow \mathbb{R}$, that is
a symmetric and positive definite function, and a so called regularisation function $\phi_\lambda$. The estimator then has the form
\begin{align}
    f_{\lambda}:=\phi_{\lambda}\left(\widehat{\Sigma}\right) \widehat{\mathcal{S}}^* \mathbf{y}, \label{kernel Method}
\end{align}
where  $\widehat{\mathcal{S}}^* \mathbf{y}:= \frac{1}{n}\sum_{i=1}^n y_i K_{x_i},\,\,$  $ \widehat{\Sigma}:=\frac{1}{n} \sum_{j=1}^{n}\left\langle\cdot, K_{x_{j}}\right\rangle_{\mathcal{H}} K_{x_{j}} \,$,  $K_x \coloneqq K(x,.)$ and $\mathcal{H}$ denotes the reproducing kernel Hilbert space (RKHS) of $K$. \cite{Muecke2017op.rates} established optimal rates for kernel methods of the above form. The idea of this estimator is, when the sample size $n$ is large, the function $\widehat{\mathcal{S}}^* \mathbf{y}= \frac{1}{n}\sum_{i=1}^n y_i K_{x_i}\in \mathcal{H}$ is a good approximation of its mean $\Sigma g_{\rho}=\int_{\mathcal{X}} g_{\rho}(x) K_{x} d \rho_{X}$. Hence the spectral algorithm (\ref{kernel Method}) produces a good estimator $f_{\lambda}$, if $\phi_{\lambda}\left(\widehat{\Sigma}\right)$ is an approximate inverse of $\Sigma$.
To motivate RFA we now consider the following examples. The probably most common example for explicit regularisation is KRR:
\begin{align}
f_\lambda(x)=\sum_{i=1}^n \alpha_i K\left(x_i, x\right), \quad \alpha=(\mathbf{K}+\lambda n I)^{-1} y, \label{KRR}
\end{align}
where $\mathbf{K}$ denotes the kernel gram matrix $\mathbf{K}_{i,j}=K(x_i,x_j)$. Note that this estimator can be obtained from \eqref{kernel Method} by choosing $\phi_\lambda(t)=\frac{1}{t+\lambda}$ \cite{Muecke2017op.rates}. In the above formula \eqref{KRR} the estimator has computational costs of order $O(n^3)$ since we need to calculate the inverse of an $n$ by $n$ matrix. However, if we assume to have a inner product kernel $K_M(x,x')=\Phi_M(x)^\top \Phi_M(x')$, where $\Phi_M$ is a feature map of dimension $M$, the computational costs can be reduced to $O(nM^2+M^3)$ \cite{features}. To also give an example of implicit regularization we here analyse an acceleration method, namely the Heavyball method which can also be derived from \eqref{kernel Method} \cite{pagliana2019implicit} and is closely related to the normal gradient descent algorithm but has an additional momentum term:
\begin{align}
f_{t+1} &= f_t - \frac{\alpha}{n}\sum_{j=1}^n (f_t(x_j ) - y_j) K(x_j , \cdot)  + \beta( f_t - f_{t-1}) \;, \label{HB}
\end{align}
where $\alpha>0,\beta \geq 0$ describe the step-sizes. So in each iteration we have to update our estimator $f_t(x_j)$ for all data points. This results in a computational cost of order $O(tn^2)$.  However if we again assume to have a inner product kernel $K_M(x,x')=\Phi_M(x)^\top \Phi_M(x')$ we can  use theory of RKHS. Recall that the RKHS of $K_M$ can be expressed as 
$$\mathcal{H}_M=\{h:\mathcal{X}\rightarrow\mathbb{R}| \,\,\exists \,\theta\in\mathbb{R}^M\,\,\,s.t.\,\,\, h(x)= \Phi_M(x)^\top \theta\}$$ 
(see for example \cite{Ingo}). Since $K_M \in \mathcal{H}_M$ and therefore all iterations $f_t \in \mathcal{H}_M $, there exists some $\theta_t \in \mathbb{R}^M$ such that $f_t(x)=\Phi_M(x)^\top \theta_t$.  This implies that instead of running \eqref{HB} it is enough to update only the parameter vector:
\begin{align}
\heavy_{t+1} &= \heavy_t - \frac{\alpha}{n}\sum_{j=1}^n (\Phi_M(x_i)^\top \theta_t - y_j)  \Phi_M(x_i) + \beta( \heavy_t - \heavy_{t-1})\;. \label{paramGD}
\end{align}
The computational cost of the above algorithm \eqref{HB} is therefore reduced from $O(tn^2)$ to $O(tnM)$.
The basic idea of RFA is now to consider kernels which can be approximated by an inner product \cite{NIPS2007_013a006f}:

\begin{align}
K_\infty(x,y)\approx K_M(x,y):=\sum_{i=1}^p \Phi_M^{(i)}(x)^\top \Phi^{(i)}_M(y), \label{kernelapprox}
\end{align}

where $\Phi_M^{(i)}: \mathcal{X} \rightarrow \mathbb{R}^M$ ,  $\Phi_M^{(i)}(x)=M^{-1/2}(\varphi^{(i)}(x,\omega_1), \dots, \varphi^{(i)}(x,\omega_M))$ is a finite dimensional feature map and $\varphi^{(i)}:\mathcal{X}\times\Omega\rightarrow \mathbb{R}$ with some probability space $(\Omega,\pi)$. More precisely this paper investigates RFA for kernels $K$ which have an integral representation of the form 

\begin{align}
K_\infty(x,y)=\sum_{i=1}^p \int_\Omega \varphi^{(i)}(x,\omega)  \varphi^{(i)}(y,\omega) d\pi(\omega). \label{kernel}
\end{align}

Note that there are a large variety of standard kernels of the form \eqref{kernel} which can be approximate by \eqref{kernelapprox}. For example, the Linear kernel, the Gaussian kernel \cite{features} or  Tangent kernels \cite{domingos2020model}. In contrast to \cite{features}, we added an additional sum over different feature maps $\Phi_M^{(i)}$, for a more general setting and to cover a special case of Tangent kernels namely the Neural-Tangent Kernel  (NTK) \cite{jacot2018neural} which provided a better understanding of neural networks in recently published papers [Paper2], \cite{nitanda2020optimal, Li21, Munteanu22, Oymak}. For one "hidden layer" the NTK is defined as 
\begin{align}
K_\infty\left(x, x^{\prime}\right) \coloneqq \int_{\Omega}\sigma\left(\omega^\top x\right) \sigma\left(\omega^\top x^{\prime}\right)+\tau^{2}\left(x^{\top} x^{\prime}+\gamma^{2}\right) \sigma^{\prime}\left(\omega^\top x\right) \sigma^{\prime}\left(\omega^\top x^{\prime}\right)d\pi(\omega),\label{NTK}
\end{align}
where $\tau, \gamma \in \mathbb{R}$ and $\sigma$ defines the so called activation function. According to our setting the NTK from above can be recovered from \eqref{kernel} by setting $p=d+2$ where $d$ denotes the input dimension and 
 $\varphi^{(i)}(x,\omega)= \tau x^{(i)}\sigma'\left(\omega^\top x\right)$ for $i \in [d]$ and $\varphi^{(d+1)}(x,\omega)= \sigma\left(\omega^\top x\right)$, $\varphi^{(d+2)}(x,\omega)= \tau\gamma\sigma'\left(\omega^\top x\right).$

\subsection{Kernel-induced operators and spectral regularization functions}

In this subsection, we specify the mathematical background of regularized learning.  It essentially repeats the setting in \cite{Muecke2017op.rates} in summarized form.
First we introduce kernel induced operators and then recall basic definitions of linear regularization methods based on spectral theory for self-adjoint linear operators. These are standard methods for finding stable
solutions for ill-posed inverse problems. Originally, these methods were developed in the
deterministic context (see \cite{engl1996regularization}). Later on, they have been applied to probabilistic problems in machine learning (see, e.g., \cite{Caponetto} or \cite{Muecke2017op.rates}).

Recall that $\mathcal{H}_M$ denotes the RKHS of the kernel $K_M$ defined in \eqref{kernelapprox}.  
We denote by $\mathcal{S}_M : \cH_M \hookrightarrow  L^2(\cX , \rho_X)$ the inclusion of $\cH_M$ into $L^2(\cX , \rho_X)$ for $M \in \mathbb{N}\cup \infty$.
The adjoint operator $\cS^{*}_M: L^{2}(\mathcal{X}, \rho_X) \longrightarrow \mathcal{H}_{M}$ is identified as
$$
\cS^{*}_M g=\int_{\mathcal{X}} g(x) K_{M,x} \rho_X(d x)
$$
where $K_{M,x}$ denotes the element of $\mathcal{H}_{M}$ equal to the function $t \mapsto K_M(x, t)$. The covariance operator $\Sigma_M: \mathcal{H}_{M} \longrightarrow \mathcal{H}_{M}$ and the kernel integral operator $\mathcal{L}_M: L^2(\cX , \rho_X) \to L^2(\cX , \rho_X) $ are given by
\begin{align*}
   \Sigma_Mf&\coloneqq \cS^*_M\cS_M f = \int_{\mathcal{X}}\left\langle f, K_{M,x}\right\rangle_{\mathcal{H}_{M}} K_{M,x} \rho_X(d x)\\ 
   \mathcal{L}_M f&\coloneqq \cS_M \cS^*_M f = \int_{\mathcal{X}} f(x) K_{M,x} \rho_X(d x)
\end{align*}

which can be shown to be positive, self-adjoint, trace class (and hence is compact).
Here $K_{M,x}$ denotes the element of $\mathcal{H}_{M}$ equal to the function $t \mapsto K_M(x, t)$. The
empirical versions of these operators, corresponding formally to taking the empirical distribution of $\rho_X$ in the above formulas, are given by

\begin{center}
\begin{align*}
&\widehat{\cS}_{M}: \mathcal{H}_{M} \longrightarrow \mathbb{R}^{n},  &&\left(\widehat{\cS}_{M} f\right)_{j}=\left\langle f, K_{M,x_{j}}\right\rangle_{\mathcal{H}_{M}}, \\
&\widehat{\cS}_{M}^{*}: \mathbb{R}^{n} \longrightarrow \mathcal{H}_{M}, && \widehat{\cS}_{M}^{*} \mathbf{y}=\frac{1}{n} \sum_{j=1}^{n} y_{j} K_{M,x_{j}}, \\
&\widehat{\Sigma}_{M}:=\widehat{\cS}_{M}^{*} \widehat{\cS}_{M}: \mathcal{H}_{M} \longrightarrow \mathcal{H}_{M},&& \widehat{\Sigma}_{M}=\frac{1}{n} \sum_{j=1}^{n}\left\langle\cdot, K_{M,x_{j}}\right\rangle_{\mathcal{H}_{M}} K_{M,x_{j}}.
\end{align*}
\end{center}

Further let the numbers $\mu_{j}$ are the positive eigenvalues of $\Sigma_\infty$ satisfying $0<\mu_{j+1} \leq \mu_{j}$ for all $j>0$ and $\mu_{j} \searrow 0$.

\begin{definition}[Regularization function] Let $\phi :(0,1]\times [0,1]\rightarrow\mathbb{R}$ be a function and write $\phi_\lambda=\phi(\lambda,.)$. The family $\{\phi_\lambda\}_\lambda$
is called regularisation function, if the following condition holds:

\begin{itemize}
    \item[(i)] There exists a constant $D<\infty$ such that for any $0<\lambda \leq 1$
\begin{align}
\sup _{0<t<1}|t\phi_{\lambda}(t)| \leq D . \label{def.phi}
\end{align}

    \item[(ii)]There exists a constant $E<\infty$ such that for any $0<\lambda \leq 1$
$$
\sup _{0 < t \leq 1}\left|\phi_{\lambda}(t)\right| \leq \frac{E}{\lambda}.
$$
    \item[(iii)]Defining the residual $r_{\lambda}(t):=1-\phi_{\lambda}(t) t$, there exists a constant $c_{0}<\infty$ such that for any $0<\lambda \leq 1$
\begin{align}
\sup _{0 < t \leq 1}\left|r_{\lambda}(t)\right| \leq c_{0}. \label{residual}
\end{align}

\end{itemize}
\end{definition}

It has been shown in e.g. Gerfo et al. (2008), Dicker et al. (2017), Blanchard and Mücke (2017) that attainable learning rates are essentially linked with the qualification of the regularization $\left\{\phi_{\lambda}\right\}_{\lambda}$, being the maximal $\nu$  such that for any $q\in[0,\nu]$ and for any $0<\lambda \leq 1$

\begin{align}
\sup _{0 < t \leq 1}\left|r_{\lambda}(t)\right| t^{q} \leq c_{q} \lambda^{q}, \label{c_r}
\end{align}

for some constant $c_{q}>0$.


\section{Main Results}
\label{sec:main-results}

\subsection{Assumptions and Main Results}

In this section we formulate our assumptions and state our main results. 

\begin{assumption}[Data Distribution]
\label{ass:input}
 There exists positive constants $Q$ and $Z$ such that for all $l \geq 2$ with $l \in \mathbb{N}$,
$$
\int_{\mathcal{Y}}|y|^l d \rho(y \mid x) \leq \frac{1}{2} l ! Z^{l-2} Q^2
$$
$\rho_X$-almost surely. The above assumption is very standard in statistical learning theory. It is for example satisfied if $y$ is bounded almost surely. Obviously, this assumption implies that the regression function $g_\rho$ is bounded almost surely, as
$$
\left|g_\rho(x)\right| \leq \int_{\mathbb{R}}|y| d \rho(y \mid x) \leq\left(\int_{\mathbb{R}}|y|^2 d \rho(y \mid x)\right)^{\frac{1}{2}} \leq Q
$$

\end{assumption}

\begin{assumption}[Kernel]
\label{ass:kernel}
Assume that the kernel $K_\infty$ has an integral representation of the form \eqref{kernel} with
$\sum_{i=1}^p|\varphi^{(i)}(x,\omega)|^2\leq \kappa^2$ almost surely. 
\end{assumption}

\begin{assumption}[Source Condition]
\label{ass:source}
Let $R>0$, $r>0$. Denote by $\mathcal{L}_\infty:  L^2(\cX , \rho_X)\to  L^2(\cX , \rho_X)$ the kernel integral operator associated to $K_\infty$. We assume 
\begin{align}
g_\rho  = \mathcal{L}_\infty^r h \;, \label{hsource}
\end{align}
for some $h \in L^2(\cX , \rho_X)$, satisfying $||h||_{L^2} \leq R$ . 
\end{assumption} 
This assumption characterizes the hypothesis space and relates to the regularity of the regression function $g_\rho$. The bigger $r$ is, the smaller the hypothesis space is, the stronger the assumption is, and the easier the learning problem is, as $\mathcal{L}^{r_{1}}\left(L_{\rho_{X}}^{2}\right) \subseteq \mathcal{L}^{r_{2}}\left(L_{\rho_{X}}^{2}\right)$ if $r_{1} \geq r_{2}$. 
The next assumption relates to the capacity of the hypothesis space.
\begin{assumption}[Effective Dimension]
\label{ass:dim} For some $b \in[0,1]$ and $c_{b}>0, \Sigma_\infty$ satisfies
\begin{align}
\mathcal{N}_{\mathcal{L}_{\infty}}:=\operatorname{tr}\left(\mathcal{L}_\infty(\mathcal{L}_\infty+\lambda I)^{-1}\right) \leq c_{b} \lambda^{-b}, \quad \text { for all } \lambda>0 \label {effecDim}
\end{align}
and further we assume that $2r+b>1$.
\end{assumption}
The left hand-side of (\ref{effecDim}) is called effective dimension or degrees of freedom \cite{Caponetto}. It is related to covering/entropy number conditions, see \cite{Ingo}. The condition (\ref{effecDim}) is naturally satisfied with $b=1$, since $\Sigma$ is a trace class operator which implies that its eigenvalues $\left\{\mu_{i}\right\}_{i}$ satisfy $\mu_{i} \lesssim i^{-1}$. Moreover, if the eigenvalues of $\Sigma$ satisfy a polynomial decaying condition $\mu_{i} \sim i^{-c}$ for some $c>1$, or if $\Sigma$ is of finite rank, then the condition (\ref{effecDim}) holds with $b=1 / c$, or with $b = 0$. The case $b = 1$ is refereed as the capacity independent case. A smaller $b$ allows deriving faster convergence rates for the studied algorithms. The assumption $2r+b>1$ refers to easy learning problems and if $2r+b\leq 1$ one speaks of hard learning problems \cite{pillaudvivien2018statistical}. In this paper we only investigate easy learning problems and leave the question, how many features $M$ are needed to obtain optimal rates in hard learning problems \cite{Lin_2020}, open for future work.
\\

We now derive a generalisation bound of the excess risk $\|g_\rho-\mathcal{S}_M f_\lambda^M\|_{L^2(\rho_x)}$ with respect to our RFA estimator,
\begin{align*}
   f_\lambda^M &\coloneqq \phi_\lambda(\widehat\Sigma_M) \widehat{\cS}_{M}^{*} y \,.
\end{align*}
The main idea of our proof is based on a bias-variance type decomposition: Further introducing

\begin{align*}
   f_\lambda^*&\coloneqq \mathcal{S}^*_M\phi_\lambda(\mathcal{L}_M) g_{\rho},
\end{align*}
we write 
\begin{align}
\|g_\rho-\mathcal{S}_M f_\lambda^M\|_{L^2(\rho_x)}&\leq \|g_\rho-\mathcal{S}_Mf_\lambda^*\|_{L^2(\rho_x)} + \|\mathcal{S}_Mf_\lambda^*-\mathcal{S}_M f_\lambda^M\|_{L^2(\rho_x)}\\[7pt]
&=: \text{ BIAS } + \text{ VARIANCE }. \label{excessrisk}
\end{align}

We bound the bias and variance part separately in Proposition \ref{mainprop2} and \ref{mainprop} to obtain the following theorem.

\begin{theorem}
\label{theo1}
Provided the Assumptions \ref{ass:input} ,\ref{ass:kernel} , \ref{ass:source}, \ref{ass:dim}  we have for 
$\lambda=C n^{-\frac{1}{2r+b}}\log^3(2/\delta)$ and $\delta\in(0,1)$ with probability at least $1-\delta$,  
\begin{align*}
\|g_\rho-\mathcal{S}_M f_\lambda^M\|_{L^2(\rho_x)}\leq \|g_\rho-\mathcal{S}_Mf_\lambda^*\|_{L^2(\rho_x)} + \|\mathcal{S}_Mf_\lambda^*-\mathcal{S}_M f_\lambda^M\|_{L^2(\rho_x)}\leq \bar{C}n^{-\frac{r}{2r+b}}\log^{3r+1}\left(\frac{18}{\delta}\right)
\end{align*}
as long as $\nu \geq 0.5+r\vee 1$,

\begin{align*}
M\geq \tilde{C}\log(n) \cdot \begin{cases}
n^{\frac{1}{2r+b}}& r\in\left(0,\frac{1}{2}\right)\\
n^{\frac{1+b(2r-1)}{2r+b}}  & r\in\left[\frac{1}{2},1\right] \\
n^{\frac{2r}{2r+b}} & r \in(1,\infty)\,\\
\end{cases}\\
\end{align*}
and $n\geq n_0:= e^{\frac{2r+b}{2r+b-1}}$, where the constants $C,\tilde{C}$ and $\bar{C}$ are independent of $n,M,\lambda$ and can be found in section \ref{I}.
\end{theorem}
If we can not make any assumption on the effective dimension i.e. assuming the worst case $b=1$ we obtain the following corollary. 

\begin{corollary}
Provided the Assumptions \ref{ass:input} ,\ref{ass:kernel} , \ref{ass:source}, with $r=0.5$ we have for 
$\lambda=C n^{-\frac{1}{2}}\log^3(2/\delta)$ and $\delta\in(0,1)$ with probability at least $1-\delta$,  
\begin{align*}
\|g_\rho-\mathcal{S}_M f_\lambda^M\|_{L^2(\rho_x)}^2\leq  \bar{C}^2n^{-\frac{1}{2}}\log^{5}\left(\frac{18}{\delta}\right)
\end{align*}
as long as $\nu \geq 0.5+r\vee 1$, $M\geq \tilde{C}\log(n) \cdot n^\frac{1}{2}$ and $n\geq 8$.
\end{corollary}

\newpage

\section{Numerical Illustration}
\label{sec:numerics}

We analyze the behavior of kernel GD (algorithm \eqref{paramGD} for $\beta=0$) with the RF of the NTK kernel \eqref{NTK}. In our simulations we used $n=5000$ training and test data points from a standard normal distributed data set with input dimension $d=1$ and a subset of the SUSY\footnote{ https://archive.ics.uci.edu/ml/datasets/SUSY} classification data set with input dimension $d=14$. The measures we show in the following simulation are an average over 50 repetitions of the algorithm.  Our theoretical analysis suggests that only a number of RF of the order of $M = O(\sqrt{n}\cdot d)$\footnote{ The linear factor of $d$ is hidden in the constants of our results and can be found in the proof section.} suffices to gain optimal learning properties. Indeed in Figure \ref{figure1} we can observe for both data sets that over a certain threshold of the order $M = O(\sqrt{n}\cdot d)$, increasing the number of RF does not improve the test error of our algorithm. 

\begin{figure}[h]
\label{figure1}
\centering
\includegraphics[width=0.3\columnwidth, height=0.23\textheight]{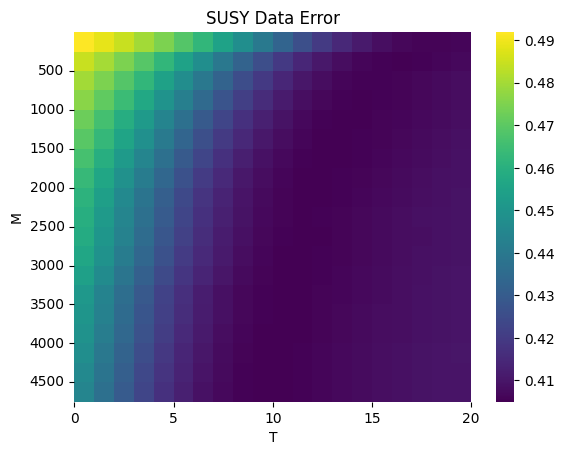}
\includegraphics[width=0.3\columnwidth, height=0.23\textheight]{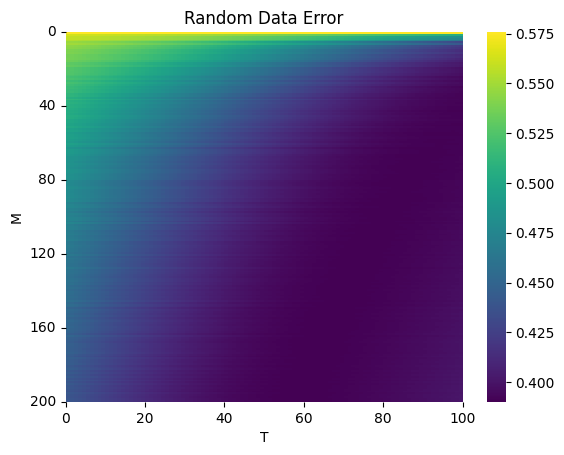}
\caption{Heat plot of the testerror for different numbers of RF $M$ and interations $T$.\\ 
{\bf Left:} Error of SUSY data set.
{\bf Right:} Error of random data set.}
\end{figure}

\newpage

\bibliographystyle{alpha}
\bibliography{bib_iteration}


\newpage

\appendix
\section{Appendix}
The proof section is organized as follows. In Appendix I we give the proofs of our main results, in Appendix II we prove some technical inequalities and Appendix III contains all the needed concentration inequalities. 

For the proofs we will use the following shortcut notations. For any Operator $A$ and $\lambda>0$ we set $A_\lambda:= A+\lambda I$ where $I$ denotes the identity operator and for any function $g$ we define the vector $\bar{g}=(g(x_1), \dots , g(x_n)) \in\mathbb{R}^n$.
 
\subsection{Appendix I}
\label{I}
To prove the following statements we need to condition on a couple of events: 
\vspace{0.5cm}\\$\begin{aligned}\vspace{0.9cm}\label{events}
&E_1=\left\{\left\|\Sigma_{M,\lambda}^{-\frac{1}{2}}\left(\widehat{\Sigma}_{M}-\Sigma_{M}\right) \Sigma_{M,\lambda}^{-\frac{1}{2}}\right\|\leq\frac{4 \kappa^2 \beta_M}{3n \lambda}+\sqrt{\frac{2 \kappa^2 \beta_M}{n\lambda}}\right\},  &&\hspace{-0.8cm}\beta_M=\log \frac{4 \kappa^2(\mathcal{N}_{\mathcal{L}_M}(\lambda)+1)}{\delta\|\mathcal{L}_M\|},\\[7pt]
&E_2=\left\{\left\|\mathcal{L}_{\infty,\lambda}^{-\frac{1}{2}}(\mathcal{L}_M-\mathcal{L}_\infty)\mathcal{L}_{\infty,\lambda}^{-\frac{1}{2}}\right\|\leq \frac{4 \kappa^2 \beta_\infty}{3M \lambda}+\sqrt{\frac{2 p\kappa^2 \beta_\infty}{M\lambda}}\right\}, &&\hspace{-0.8cm}\beta_\infty=\log \frac{4 \kappa^2(\mathcal{N}_{\mathcal{L}_\infty}(\lambda)+1)}{\delta\|\mathcal{L}_\infty\|}  ,\\
&E_3=\left\{\left\|\Sigma_{M,\lambda}^{-\frac{1}{2}}\left(\widehat{\Sigma}_{M}-\Sigma_{M}\right)\right\|_{HS}\leq  \left(\frac{2\kappa}{\sqrt{\lambda} n}+\sqrt{\frac{ 4\kappa^2 \mathcal{N}_{\mathcal{L}_M}(\lambda) }{ n}}\right)\log \frac{2}{\delta}\right\},\\
&E_4=\left\{\left\|\mathcal{L}_{\infty,\lambda}^{-\frac{1}{2}}(\mathcal{L}_M-\mathcal{L}_\infty)\mathcal{L}_{\infty,\lambda}^{-\frac{1}{2}}\right\|_{HS}\leq  \left(\frac{4\kappa^2}{\lambda M}+\sqrt{\frac{ 4\kappa^2 \mathcal{N}_{\mathcal{L}_\infty}(\lambda) }{\lambda M}}\right)\log \frac{2}{\delta}\right\},\\
&E_5=\left\{\left\|\mathcal{L}_{\infty,\lambda}^{-\frac{1}{2}}\left(\mathcal{L}_M-\mathcal{L}_\infty\right)\right\| \leq \left(\frac{2 
\kappa}{\sqrt{\lambda}M}+\sqrt{\frac{4 \kappa^2 \mathcal{N}_{\mathcal{L}_\infty}(\lambda) }{M}}\right)\log \frac{2}{\delta}\right\},\\
&E_6=\left\{\left\|\mathcal{L}_\infty-\mathcal{L}_M\right\|_{H S} \leq \left(\frac{2 \kappa^2}{M} + \frac{2 \kappa^2}{\sqrt{M}} \right)\log \frac{2}{\delta}\right\}\,,\\
&E_7=\left\{\left\|\widehat{\Sigma}_{M}-\Sigma_{M}\right\|_{HS} \leq \left(\frac{2 \kappa^2}{n} + \frac{2 \kappa^2}{\sqrt{n}} \right)\log \frac{2}{\delta}\right\}\,\\
&E_8= \left\{\left\|\Sigma_{M,\lambda}^{-\frac{1}{2}}\widehat{\mathcal{S}}_{M}^{*}\left(y-\bar{g}_\rho\right)\right\|_{\mathcal{H}_M} \leq \left(\frac{4QZ\kappa}{\sqrt{\lambda}n}+\frac{4Q\sqrt{\mathcal{N}_{\mathcal{L}_M}(\lambda)}}{\sqrt{n}}\right) \log \frac{2}{\delta}\right\} \,\\
&E_9= \left\{\left|\frac{1}{n}\left\|\bar{g}_\rho-\widehat{\mathcal{S}}_M f_\lambda^*\right\|_2^2-\left\|g_\rho-\mathcal{S}_M f_\lambda^*\right\|_{L^2(\rho_x)}^2\right| \leq  2\left(\frac{B_\lambda}{n}+\frac{V_\lambda}{\sqrt{n}}\right) \log \frac{2}{\delta}\right\},
\end{aligned}$\vspace{0.3cm}\\
where $B_\lambda:=4(Q^2+ C_{\kappa,R,D}^2 \,\lambda^{-2(\frac{1}{2}-r)^+})$,  $V_\lambda:=\sqrt{2}(Q+C_{\kappa,R,D}  \,\lambda^{-(\frac{1}{2}-r)^+})\left\|g_\rho-\mathcal{S}_M f_\lambda^*\right\|_{L^2(\rho_x)} $ and  $C_{\kappa,R,D}= 2 \kappa^{2r+1} R D$. In section \ref{III} we prove that all of the above events occur with probability at least $1-\delta$.

First we start bounding the bias part of our excess risk \eqref{excessrisk}. 

\begin{proposition}
\label{mainprop2}
Given the Assumptions \ref{ass:input} ,\ref{ass:kernel} , \ref{ass:source}, \ref{ass:dim}  and premise that the  events $E_2, E_5, E_6$ from above occur, then we have for
\begin{align*}
M&\geq 
\begin{cases}
\frac{8 p\kappa^2 \beta_\infty}{\lambda}\vee C_{\delta,\kappa} & r\in\left(0,\frac{1}{2}\right)\\
\frac{(8 p\kappa^2 \beta_\infty)\vee C_1^{\frac{1}{r}}}{\lambda}\vee \frac{C_2}{\lambda^{1+b(2r-1)} } \vee C_{\delta,\kappa} & r\in\left[\frac{1}{2},1\right] \\
\frac{C_3}{\lambda^{2r}}\vee C_{\delta,\kappa} & r \in(1,\infty),\\
\end{cases}
\end{align*}
that the bias term can be bounded by
\begin{align*}
\|g_\rho-\mathcal{S}_Mf_\lambda^*\|_{L^2(\rho_x)} \leq 3 R c_{r\vee 1}\lambda^{r}.
\end{align*}

\end{proposition}

\begin{proof}
We use from Assumption \ref{ass:source} that  $g_\rho=\mathcal{L}_\infty^r h$  with $\|h\|_{L^2(\rho_x)}\leq R$  to obtain, 
\begin{align}
\|g_\rho-\mathcal{S}_Mf_\lambda^*\|_{L^2(\rho_x)} &= \left\|\left(\mathcal{L}_M\phi_\lambda(\mathcal{L}_M)-I\right)\mathcal{L}_\infty^r h\right\|_{L^2(\rho_x)} \leq R \left\|r_\lambda(\mathcal{L}_M)\mathcal{L}_\infty^r\right\|\, ,
\end{align}

where $r_\lambda$ denotes the residual polynomial from \eqref{residual}. For the last term we have 
\begin{align*}
R\left\| r_\lambda(\mathcal{L}_M) \mathcal{L}_\infty^r\right\|&\leq R \left\| r_\lambda(\mathcal{L}_M) \mathcal{L}_{M,\lambda}^{(r\vee 1)}\right\|\left\|\mathcal{L}_{M,\lambda}^{-(r\vee1)}\mathcal{L}_{\infty,\lambda}^{r}\right\|\\
&\leq 3 R c_{r\vee 1}\lambda^{r},
\end{align*}
where we used for the last inequality that from \eqref{c_r} we have $\left\| r_\lambda(\mathcal{L}_M) \mathcal{L}_{M,\lambda}^{(r\vee 1)}\right\|\leq c_{r\vee 1} \lambda^{(r\vee 1)}$ and given the events $E_2, E_5, E_6$ and the conditions on $M$ we have from Proposition \ref{OPbound7}  $\left\|\mathcal{L}_{M,\lambda}^{-(r\vee1)}\mathcal{L}_{\infty,\lambda}^{r}\right\| \leq 3\lambda^{-(1-r)^+}$.

\end{proof}
Now we want to bound the variance term. To do so we first need the following technical proposition.

\begin{proposition}
\label{T2I} 
Given the Assumptions \ref{ass:input} ,\ref{ass:kernel} , \ref{ass:source}, \ref{ass:dim}  and premise that the  events $E_1-E_9$ from above occur, then we have for any $s\in[0,0.5]$\\
\begin{align*}
&a) \,\,\,\left\|\Sigma_M^{\frac{1}{2}-s} \phi_\lambda(\widehat\Sigma_M) \widehat{\mathcal{S}}_{M}^{*}\left(y-\widehat{\mathcal{S}}_M f_\lambda^*\right)\right\|_{\mathcal{H}_M}\leq 12D\left(\log\frac{2}{\delta}+ R c_{r\vee 1}\right)\lambda^{r-s},\\[5pt]
&b)\,\,\,\, \left\|\Sigma_M^{\frac{1}{2}-s}r_\lambda(\widehat\Sigma_M)f_\lambda^*\right\|_{\mathcal{H}_M}\leq 12DR c_{\frac{1}{2}+r   }    \lambda^{r-s},
\end{align*}
as long as 
\vspace{-0.2cm}
\begin{align*}
M&\geq 
\begin{cases}
\frac{8 p\kappa^2 \beta_\infty}{\lambda}\vee C_{\delta,\kappa} & r\in\left(0,\frac{1}{2}\right)\\
\frac{(8 p\kappa^2 \beta_\infty)\vee C_1^{\frac{1}{r}}}{\lambda}\vee \frac{C_2}{\lambda^{1+b(2r-1)} } \vee C_{\delta,\kappa} & r\in\left[\frac{1}{2},1\right] \\
\frac{C_3}{\lambda^{2r}}\vee C_{\delta,\kappa} & r \in(1,\infty),\\
\end{cases}\\
 n&\geq 
\begin{cases}
\eta_1\vee\eta_2\vee\eta_3\vee\eta_4 & r\in(0,\frac{1}{2}],\\
\eta_1\vee\eta_2\vee\eta_3\vee\eta_4 \vee\eta_5 \vee\eta_6 & r>\frac{1}{2},
\end{cases}
\end{align*}

where  $C_1=2(4\kappa\log\frac{2}{\delta})^{2r-1}(8p\kappa^2\beta_\infty)^{1-r}$ ,  $C_2=4(4c_b\kappa^2\log\frac{2}{\delta})^{2r-1}(8p\kappa^2\beta_\infty)^{2-2r}$, \\ $C_3:= 4\kappa^4C_{\kappa,r}^2\log^2\frac{2}{\delta}$ ,  $C_{\delta,\kappa}= 8\kappa^4\|\mathcal{L}_\infty\|^{-1}\log^2 \frac{2}{\delta}$ , $\beta_\infty=\log \frac{4 \kappa^2(\mathcal{N}_{\mathcal{L}_\infty}(\lambda)+1)}{\delta\|\mathcal{L}_\infty\|}$, $C_{\kappa,r}$ from Proposition \ref{ineq1}, 
\begin{center}\vspace{-0.5cm}
$\begin{aligned}
&\eta_1:= \frac{8\kappa^2 \tilde\beta}{\lambda}, 
&& \eta_2:=\frac{8QZ\kappa}{\lambda^{r+\frac{1}{2}}}, \\
&\eta_3:=\frac{128Q^2\left(1+2\log\frac{2}{\delta}\right)\mathcal{N}_{\mathcal{L}_{\infty}}(\lambda)}{\lambda^{2r}}, &&\eta_4:=\frac{72 R^2 c_{r\vee 1}^2\left(Q^2+C_{\kappa,R,D}^2\right)}{\lambda^{2r+(1-2r)^+}},\\
&\eta_5 = \frac{100\kappa^2 \mathcal{N}_{\mathcal{L}_{\infty}}(\lambda)\log^3 \frac{2}{\delta}}{\lambda} , 
&&\eta_6= \frac{8C_{\kappa,r}^2 \kappa^4 \log^2\frac{2}{\delta}}{\lambda^{2r}}
\end{aligned}$ \\\vspace{0.2cm}
\end{center}
and  $\tilde{\beta}:= \log \frac{4 \kappa^2(\left(1+2\log\frac{2}{\delta}\right)4\mathcal{N}_{\mathcal{L}_{\infty}}(\lambda)+1)}{\delta\|\mathcal{L}_\infty\|}$,  $C_{\kappa,R,D}=2 \kappa^{2r+1} R D$.

\end{proposition}

\begin{proof}
\textbf{$a)$} We start with the following decomposition

\begin{align}
\nonumber \Biggl\|\Sigma_M^{\frac{1}{2}-s}& \phi_\lambda(\widehat\Sigma_M) \widehat{\mathcal{S}}_{M}^{*}\left(y-\widehat{\mathcal{S}}_M f_\lambda^*\right)\Biggr\|_{L^2(\rho_x)} \\
\nonumber \leq& \Biggl\|\Sigma_M^{\frac{1}{2}-s}\phi_\lambda(\widehat\Sigma_M) \Sigma_{M,\lambda}^{\frac{1}{2}}\Biggr\|\left\|\Sigma_{M,\lambda}^{-\frac{1}{2}}\widehat{\mathcal{S}}_{M}^{*}\left(y-\widehat{\mathcal{S}}_M f_\lambda^*\right)\right\|_{\mathcal{H}_M} \\
=& I\cdot II. \label{I.II}
\end{align}

\begin{itemize}

\item[$I)$] Provided the events $E_1,E_2$, $E_4, E_5, E_6$ hold true we have from Proposition \ref{OPbound6} , \\$\left\|\widehat{\Sigma}_{M,\lambda}^{-\frac{1}{2}}\Sigma_{M,\lambda}^{\frac{1}{2}}\right\| \leq 2$ . Using this inequality we obtain
\begin{align*}
\Biggl\|\Sigma_M^{\frac{1}{2}-s} \phi_\lambda(\widehat\Sigma_M) \Sigma_{M,\lambda}^{\frac{1}{2}}\Biggr\|&\leq \lambda^{-s}\Biggl\|\Sigma_{M,\lambda}^{\frac{1}{2}} \phi_\lambda(\widehat\Sigma_M) \Sigma_{M,\lambda}^{\frac{1}{2}}\Biggr\|\\
&\leq \lambda^{-s}\Biggl\|\Sigma_{M,\lambda}^{\frac{1}{2}} \phi_\lambda(\widehat\Sigma_M) \Sigma_{M,\lambda}^{\frac{1}{2}}\Biggr\|\\
&\leq \lambda^{-s}\Biggl\|\widehat\Sigma_{M,\lambda} \phi_\lambda(\widehat\Sigma_M) \Biggr\|\left\|\Sigma_{M,\lambda}^{\frac{1}{2}} \widehat{\Sigma}_{M,\lambda}^{-\frac{1}{2}}  \right\|^2\leq \lambda^{-s}4 D,
\end{align*}
where $D$ is defined in \eqref{def.phi}.
\item[$II)$] For the second term we have 

\begin{align*}
&\left\|\Sigma_{M,\lambda}^{-\frac{1}{2}}\widehat{\mathcal{S}}_{M}^{*}\left(y-\widehat{\mathcal{S}}_M f_\lambda^*\right)\right\|_{\mathcal{H}_M}\\
&\leq \left\|\Sigma_{M,\lambda}^{-\frac{1}{2}}\widehat{\mathcal{S}}_{M}^{*}\left(y-\bar{g}_\rho\right)\right\|_{\mathcal{H}_M}+\left\|\Sigma_{M,\lambda}^{-\frac{1}{2}}\widehat{\mathcal{S}}_{M}^{*}\left(\bar{g}_\rho-\widehat{\mathcal{S}}_M f_\lambda^*\right)\right\|_{\mathcal{H}_M}\\
&:= i+ii
\end{align*}
For the first norm $i)$ we use the bound of event $E_8$ together with the bound of \ref{prop:effecdim2}:
$$
\mathcal{N}_{\mathcal{L}_{M}}(\lambda)\leq  \left(1+2\log\frac{2}{\delta}\right)4\mathcal{N}_{\mathcal{L}_{\infty}}(\lambda),
$$
to obtain 
\begin{align*}
\left\|\Sigma_{M,\lambda}^{-\frac{1}{2}}\widehat{\mathcal{S}}_{M}^{*}\left(y-\bar{g}_\rho\right)\right\|_{\mathcal{H}_M}&\leq \left(\frac{4QZ\kappa}{\sqrt{\lambda}n}+\frac{4Q\sqrt{\mathcal{N}_{\mathcal{L}_M}(\lambda)}}{\sqrt{n}}\right) \log \frac{2}{\delta}\\
&\leq\left(\frac{4QZ\kappa}{\sqrt{\lambda}n}+\frac{8Q\sqrt{\left(1+2\log\frac{2}{\delta}\right)\mathcal{N}_{\mathcal{L}_{\infty}}(\lambda)}}{\sqrt{n}}\right) \log \frac{2}{\delta}\\
&\leq \lambda^{r}\log \frac{2}{\delta},
\end{align*}
where we used in the last inequality that $n\geq\eta_2\vee\eta_3 :=\frac{8QZ\kappa}{\lambda^{r+\frac{1}{2}}}\vee\frac{128Q^2\left(1+2\log\frac{2}{\delta}\right)\mathcal{N}_{\mathcal{L}_{\infty}}(\lambda)}{\lambda^{2r}}$.

For the second norm $ii)$ we first use that
\begin{align*}
\left\|\Sigma_{M,\lambda}^{-\frac{1}{2}}\widehat{\mathcal{S}}_{M}^{*}\right\|^2&=
\left\|(\widehat{\mathcal{S}}_M^*\widehat{\mathcal{S}}_M+\lambda)^{-1/2}\widehat{\mathcal{S}}_M^*\right\|^2\\
&=\left\|(\widehat{\mathcal{S}}_M^*\widehat{\mathcal{S}}_M+\lambda)^{-1/2}\widehat{\mathcal{S}}_M^*\widehat{\mathcal{S}}_M(\widehat{\mathcal{S}}_M^*\widehat{\mathcal{S}}_M+\lambda)^{-1/2}\right\|\\
&=\left\|\widehat{\mathcal{S}}_M^*\widehat{\mathcal{S}}_M(\widehat{\mathcal{S}}_M^*\widehat{\mathcal{S}}_M+\lambda)^{-1}\right\|\leq1,
\end{align*}
to obtain together with the bound of event $E_9$,
\begin{align*}
&\left\|\Sigma_{M,\lambda}^{-\frac{1}{2}}\widehat{\mathcal{S}}_{M}^{*}\left(\bar{g}_\rho-\widehat{\mathcal{S}}_M f_\lambda^*\right)\right\|_{\mathcal{H}_M}\\
&\leq \frac{1}{\sqrt{n}}\left\|\bar{g}_\rho-\widehat{\mathcal{S}}_M f_\lambda^*\right\|_2\\
&\leq \sqrt{\left|\frac{1}{n}\left\|\bar{g}_\rho-\widehat{\mathcal{S}}_M f_\lambda^*\right\|_2^2-\left\|g_\rho-\mathcal{S}_M f_\lambda^*\right\|_{L^2(\rho_x)}^2\right|}+\left\|g_\rho-\mathcal{S}_M f_\lambda^*\right\|_{L^2(\rho_x)}\\
&\leq \sqrt{2\left(\frac{4\left(Q^2+ C_{\kappa,R,D}^2 \,\lambda^{-2(\frac{1}{2}-r)^+}\right)}{n}+\frac{\sqrt{2}\left(Q+C_{\kappa,R,D}  \,\lambda^{-(\frac{1}{2}-r)^+}\right)\left\|g_\rho-\mathcal{S}_M f_\lambda^*\right\|_{L^2(\rho_x)}}{\sqrt{n}}\right) \log \frac{2}{\delta}}\,\,+\\[5pt]
&\,\,\,\,\,\,\,\,\,\,\,\left\|g_\rho-\mathcal{S}_M f_\lambda^*\right\|_{L^2(\rho_x)}.
\end{align*}

From Proposition \ref{mainprop2} we further obtain
\begin{align*}
&\left\|\Sigma_{M,\lambda}^{-\frac{1}{2}}\widehat{\mathcal{S}}_{M}^{*}\left(\bar{g}_\rho-\widehat{\mathcal{S}}_M f_\lambda^*\right)\right\|_{\mathcal{H}_M}\\
&\leq \sqrt{2\left(\frac{4\left(Q^2+ C_{\kappa,R,D}^2 \,\lambda^{-(1-2r)^+}\right)}{n}+\frac{\sqrt{2}\left(Q+C_{\kappa,R,D}  \,\lambda^{-(\frac{1}{2}-r)^+}\right)3 R c_{r\vee 1}\lambda^{r}}{\sqrt{n}}\right) \log \frac{2}{\delta}}+3 R c_{r\vee 1}\lambda^{r}\\
&\leq \lambda^{r} \left(\sqrt{\log\frac{2}{\delta}}+3 R c_{r\vee 1}\right),
\end{align*}
where we used in the last inequality that $n\geq \eta_4:=\frac{72 R^2 c_{r\vee 1}^2\left(Q^2+C_{\kappa,R,D}^2\right)}{\lambda^{2r+(1-2r)^+}}$.
Therefore we have for the second term
$$
II \leq \lambda^{r} \left(\log \frac{2}{\delta}+\sqrt{\log\frac{2}{\delta}}+3 R c_{r\vee 1}\right)
$$
\end{itemize}
Plugging the bounds of $I$ and $II$ in \eqref{I.II} proves the claim.\\

\textbf{$b)$}Using Mercers theorem (see for example \cite{steinwart2008support}) we have
\begin{align*}
\left\|\Sigma_M^{\frac{1}{2}-s} r_\lambda(\widehat\Sigma_M)f_\lambda^*\right\|_{\mathcal{H}_M}&\leq\lambda^{-s}\left\| \Sigma_{M,\lambda}^{\frac{1}{2}} r_\lambda(\widehat\Sigma_M)f_\lambda^*\right\|_{\mathcal{H}_M}\\
&\leq \lambda^{-s}\left\|\widehat{\Sigma}_{M,\lambda}^{-\frac{1}{2}}\Sigma_{M,\lambda}^{\frac{1}{2}}\right\| \left\| \widehat\Sigma_{M,\lambda}^{\frac{1}{2}} r_\lambda(\widehat\Sigma_M)f_\lambda^*\right\|_{\mathcal{H}_M}\\
&\leq 2 \lambda^{-s} \left\| \widehat\Sigma_{M,\lambda}^{\frac{1}{2}} r_\lambda(\widehat\Sigma_M)f_\lambda^*\right\|_{\mathcal{H}_M},
\end{align*}
where we used Proposition \ref{OPbound6} for the last inequality. 
To continue we write out the definition of $f_\lambda^*$ to obtain
\begin{align}
\nonumber&2 \lambda^{-s}\left\| \widehat\Sigma_{M,\lambda}^{\frac{1}{2}} r_\lambda(\widehat\Sigma_M)f_\lambda^*\right\|_{L^2(\rho_x)}\\
&\leq 2 R \lambda^{-s} \left\|\widehat\Sigma_{M,\lambda}^{\frac{1}{2}} r_\lambda(\widehat\Sigma_M) \mathcal{S}_M^* \phi(\mathcal{L}_M) \mathcal{L}_\infty^r\right\|.\label{casesT2}
\end{align}
To bound the last term we need to differ between the following two cases. 
\begin{itemize}

\item CASE ($r\leq \frac{1}{2}$) : To bound the norm of \eqref{casesT2} for $r\leq\frac{1}{2}$ we start with 
\begin{align*}
&\left\|\widehat\Sigma_{M,\lambda}^{\frac{1}{2}} r_\lambda(\widehat\Sigma_M) \mathcal{S}_M^* \phi(\mathcal{L}_M) \mathcal{L}_\infty^r\right\|\\
&\leq\left\|\widehat\Sigma_{M,\lambda}^{\frac{1}{2}} r_\lambda(\widehat\Sigma_M) \mathcal{S}_M^* \phi(\mathcal{L}_M) \mathcal{L}_{M,\lambda}^{r}\right\|\left\|\mathcal{L}_{M,\lambda}^{-r}\mathcal{L}_{\infty,\lambda}^{r}\right\|\\
&= \left\|\widehat\Sigma_{M,\lambda}^{\frac{1}{2}} r_\lambda(\widehat\Sigma_M) \Sigma_{M,\lambda}^{r}\mathcal{S}_M^* \phi(\mathcal{L}_M) \right\|\left\|\mathcal{L}_{M,\lambda}^{-r}\mathcal{L}_{\infty,\lambda}^{r}\right\|.\\
&\leq  \left\|\widehat\Sigma_{M,\lambda}^{\frac{1}{2}} r_\lambda(\widehat\Sigma_M) \Sigma_{M,\lambda}^{r}\right\| \left\| \mathcal{L}_M^{\frac{1}{2}}\phi(\mathcal{L}_M) \right\|\left\|\mathcal{L}_{M,\lambda}^{-r}\mathcal{L}_{\infty,\lambda}^{r}\right\|.
\end{align*}

From Proposition \ref{ineqvolkan} we have $ \left\| \mathcal{L}_M^{\frac{1}{2}}\phi(\mathcal{L}_M) \right\| \leq D \lambda^{-\frac{1}{2}}$ and from Proposition \ref{ineq2} together with \ref{OPbound5} we have $\left\|\mathcal{L}_{M,\lambda}^{-r}\mathcal{L}_{\infty,\lambda}^{r}\right\|\leq\left\|\mathcal{L}_{M,\lambda}^{-\frac{1}{2}}\mathcal{L}_{\infty,\lambda}^{\frac{1}{2}}\right\|^{2r}\leq 2^{2r}\leq 2$ (as long as event $E_2$ holds true).  Using those bounds we obtain for \eqref{casesT2}

\begin{align}
\left\|\Sigma_M^{\frac{1}{2}-s}r_\lambda(\widehat\Sigma_M)f_\lambda^*\right\|_{\mathcal{H}_M}\leq 4DR \lambda^{-s-\frac{1}{2}}\left\|\widehat\Sigma_{M,\lambda}^{\frac{1}{2}} r_\lambda(\widehat\Sigma_M) \Sigma_{M,\lambda}^{r}\right\|. \label{TIIcase1}
\end{align}

It remains to bound $\left\|\widehat\Sigma_{M,\lambda}^{\frac{1}{2}} r_\lambda(\widehat\Sigma_M) \Sigma_{M,\lambda}^{r}\right\|$ . Using the events $E_1,E_2,E_4,E_6$ we have from Proposition \ref{OPbound6} that $\left\|\widehat{\Sigma}_{M,\lambda}^{-r} \Sigma_{M,\lambda}^{r}\right\|\leq \left\|\widehat{\Sigma}_{M,\lambda}^{-\frac{1}{2}} \Sigma_{M,\lambda}^{\frac{1}{2}}\right\|^{2r}\leq 2$. From this bound together with \eqref{c_r}
we obtain 
\begin{align*}
\left\|\widehat\Sigma_{M,\lambda}^{\frac{1}{2}} r_\lambda(\widehat\Sigma_M) \Sigma_{M,\lambda}^{r}\right\|&\leq \left\|\widehat\Sigma_{M,\lambda}^{\frac{1}{2}} r_\lambda(\widehat\Sigma_M) \widehat{\Sigma}_{M,\lambda}^{r}\right\|\left\|\widehat{\Sigma}_{M,\lambda}^{-r} \Sigma_{M,\lambda}^{r}\right\| \\
&\leq c_{\frac{1}{2}+r} \lambda^{\frac{1}{2}+r} \left\|\widehat{\Sigma}_{M,\lambda}^{-r} \Sigma_{M,\lambda}^{r}\right\|\\
&\leq 2c_{\frac{1}{2}+r} \lambda^{\frac{1}{2}+r} 
\end{align*}

Plugging the above bound into \eqref{TIIcase1} gives
\begin{align*}
\left\|\Sigma_M^{\frac{1}{2}-s} r_\lambda(\widehat\Sigma_M)f_\lambda^*\right\|_{\mathcal{H}_M}\leq 8DR c_{\frac{1}{2}+r   }     \lambda^{r-s}.
\end{align*}

\item CASE ($r>\frac{1}{2}$) :  To bound the norm of \eqref{casesT2} for $r>\frac{1}{2}$ we start similar with 
\begin{align*}
&\left\|\widehat\Sigma_{M,\lambda}^{\frac{1}{2}} r_\lambda(\widehat\Sigma_M) \mathcal{S}_M^* \phi(\mathcal{L}_M) \mathcal{L}_\infty^r\right\|\\
&\leq\left\|\widehat\Sigma_{M,\lambda}^{\frac{1}{2}} r_\lambda(\widehat\Sigma_M) \mathcal{S}_M^* \phi(\mathcal{L}_M) \mathcal{L}_{M,\lambda}^{(r\vee1)}\right\|\left\|\mathcal{L}_{M,\lambda}^{-(r\vee1)}\mathcal{L}_{\infty,\lambda}^{r}\right\|\\
&= \left\|\widehat\Sigma_{M,\lambda}^{\frac{1}{2}} r_\lambda(\widehat\Sigma_M) \Sigma_{M,\lambda}^{(r\vee1)}\mathcal{S}_M^* \phi(\mathcal{L}_M) \right\|\left\|\mathcal{L}_{M,\lambda}^{-(r\vee1)}\mathcal{L}_{\infty,\lambda}^{r}\right\|.\\
&\leq  \left\|\widehat\Sigma_{M,\lambda}^{\frac{1}{2}} r_\lambda(\widehat\Sigma_M) \Sigma_{M,\lambda}^{(r\vee1)}\right\| \left\| \mathcal{L}_M^{\frac{1}{2}}\phi(\mathcal{L}_M) \right\|\left\|\mathcal{L}_{M,\lambda}^{-(r\vee1)}\mathcal{L}_{\infty,\lambda}^{r}\right\|.
\end{align*}

From Proposition \ref{ineqvolkan} and  \ref{OPbound7} we have $ \left\| \mathcal{L}_M^{\frac{1}{2}}\phi(\mathcal{L}_M) \right\| \leq D \lambda^{-\frac{1}{2}}$ and $\left\|\mathcal{L}_{M,\lambda}^{-(r\vee1)}\mathcal{L}_{\infty,\lambda}^{r}\right\|\leq3 \lambda^{- (1-r)^+}$. Using those bounds we obtain for \eqref{casesT2}

\begin{align}
\left\|\Sigma_M^{\frac{1}{2}-s}r_\lambda(\widehat\Sigma_M)f_\lambda^*\right\|_{\mathcal{H}_M}\leq \frac{6DR}{\lambda^{\frac{1}{2}+s+(1-r)^+}} \left\|\widehat\Sigma_{M,\lambda}^{\frac{1}{2}} r_\lambda(\widehat\Sigma_M) \Sigma_{M,\lambda}^{(r\vee1)}\right\|. \label{TIIcase2}
\end{align}

It remains to bound $\left\|\widehat\Sigma_{M,\lambda}^{\frac{1}{2}} r_\lambda(\widehat\Sigma_M) \Sigma_{M,\lambda}^{(r\vee1)}\right\|$ . From \eqref{c_r} and Proposition \ref{OPbound8} we have
\begin{align*}
\left\|\widehat\Sigma_{M,\lambda}^{\frac{1}{2}} r_\lambda(\widehat\Sigma_M) \Sigma_{M,\lambda}^{(r\vee1)}\right\|&\leq \left\|\widehat\Sigma_{M,\lambda}^{\frac{1}{2}} r_\lambda(\widehat\Sigma_M) \widehat{\Sigma}_{M,\lambda}^{r}\right\|\left\|\widehat{\Sigma}_{M,\lambda}^{-(r\vee1)} \Sigma_{M,\lambda}^{(r\vee1)}\right\| \\
&\leq c_{\frac{1}{2}+r} \lambda^{\frac{1}{2}+(r\vee1)} \left\|\widehat{\Sigma}_{M,\lambda}^{-(r\vee1)} \Sigma_{M,\lambda}^{r}\right\|\\
&\leq2 c_{\frac{1}{2}+r} \lambda^{\frac{1}{2}+(r\vee1)} 
\end{align*}

Plugging the above bound into \eqref{TIIcase2} gives
\begin{align*}
\left\|\mathcal{S}_M r_\lambda(\widehat\Sigma_M)f_\lambda^*\right\|_{L^2(\rho_x)}\leq 12DR c_{\frac{1}{2}+r   }    \lambda^{r-s}.
\end{align*}

\end{itemize}
Combining the bounds of both cases proves the claim.
\end{proof}

Now we are able to bound the variance term.

\begin{proposition}
\label{mainprop}
Provided the same assumptions of Proposition \ref{T2I}, we have for any $s\in[0,0.5]$\\
\begin{align*}
\left\|\Sigma_M^{\frac{1}{2}-s} (f_\lambda^M-f_\lambda^*)\right\|_{\mathcal{H}_M}\leq \left( 12D\left(\log\frac{2}{\delta}+ R c_{r\vee 1}\right)+ 12DR c_{\frac{1}{2}+r }  \right)\lambda^{r-s}.
\end{align*}
\end{proposition}

\begin{proof}
We start with the following decomposition
\begin{align}
&\left\|\Sigma_M^{\frac{1}{2}-s} (f_\lambda^M-f_\lambda^*)\right\|_{\mathcal{H}_M} \\
&\leq\left\|\Sigma_M^{\frac{1}{2}-s}\left(\phi_\lambda(\widehat\Sigma_M) \widehat{\cS}_{M}^{*} y-\phi_\lambda(\widehat\Sigma_M) \widehat{\Sigma}_{M}f_\lambda^*\right)\right\|_{\mathcal{H}_M}+\left\|\Sigma_M^{\frac{1}{2}-s} \left(\phi_\lambda(\widehat\Sigma_M) \widehat{\Sigma}_{M}-I\right)f_\lambda^*\right\|_{\mathcal{H}_M}\\
&=\left\|\Sigma_M^{\frac{1}{2}-s}\phi_\lambda(\widehat\Sigma_M) \widehat{\cS}_{M}^{*}\left(y-\widehat{\mathcal{S}}_M f_\lambda^*\right)\right\|_{\mathcal{H}_M}+\left\|\Sigma_M^{\frac{1}{2}-s} r_\lambda(\widehat\Sigma_M)f_\lambda^*\right\|_{\mathcal{H}_M}\\
&:=I+II.
\end{align}

For $I)$ we have from Proposition \ref{T2I}  a)
$$I\leq 12D\left(\log\frac{2}{\delta}+ R c_{r\vee 1}\right)\lambda^{r-s}$$
and for $II)$ we have from Proposition \ref{T2I} b)
$$II\leq 12DR c_{\frac{1}{2}+r }  \lambda^{r-s}.$$
Combining those bounds proves the claim.
\end{proof}

\begin{theorem}
\label{theo2}
Provided all the assumptions of Proposition \ref{T2I}  we have \\
\begin{align*}
\|g_\rho-\mathcal{S}_M f_\lambda^M\|_{L^2(\rho_x)}\leq \left(3 R c_{r\vee 1}+ 12D\left(\log\frac{2}{\delta}+ R c_{r\vee 1}\right)+ 12DR c_{\frac{1}{2}+r }  \right)\lambda^{r}.
\end{align*}

\end{theorem}

\begin{proof}
We start with the following decomposition
\begin{align}
\|g_\rho-\mathcal{S}_M f_\lambda^M\|_{L^2(\rho_x)}\leq\|g_\rho-\mathcal{S}_Mf_\lambda^*\|_{L^2(\rho_x)}+\|\mathcal{S}_M f_\lambda^M-\mathcal{S}_Mf_\lambda^*\|_{L^2(\rho_x)}
:=T_1+T_2 .\label{Maindecomposition}
\end{align}

We will now bound $T_1$ and $T_2$ separately :
\begin{itemize}

\item[$T_1)$] For the first term of \eqref{Maindecomposition} we have from Proposition \ref{mainprop2}
\begin{align}
\|g_\rho-\mathcal{S}_Mf_\lambda^*\|_{L^2(\rho_x)} \leq 3 R c_{r\vee 1}\lambda^{r}\, .\label{(T_1)}
\end{align}

\item[$T_2)$] For the second norm in \eqref{Maindecomposition} we have from Mercers Theorem (see for example \cite{steinwart2008support}) and Proposition \ref{mainprop},
\begin{align*}
\|\mathcal{S}_M f_\lambda^M-\mathcal{S}_Mf_\lambda^*\|_{L^2(\rho_x)} = \left\|\Sigma_M^{\frac{1}{2}} (f_\lambda^M-f_\lambda^*)\right\|_{\mathcal{H}_M} \leq \left( 12D\left(\log\frac{2}{\delta}+ R c_{r\vee 1}\right)+ 12DR c_{\frac{1}{2}+r }  \right)\lambda^{r}.
\end{align*}
Combining this bound with the bound of $T_1$ proves the claim.

\end{itemize}

\end{proof}

\begin{proof}[Proof of Theorem \ref{theo1}]
The proof follows from \ref{theo2} . First we need to check if $\lambda = C n^{-\frac{1}{2r+b}}\log^3\frac{2}{\delta}$ for some $C>0$ fulfills the conditions of \ref{T2I} on $n$ and $M$. Using the bound of \ref{effecDim}  we have that the condition  $n\geq \eta_1\vee\eta_2\vee\eta_3\vee\eta_4 \vee\eta_6$ is fulfilled if
$$
n\geq c_1 \left(\frac{\log(\lambda^{-1})}{\lambda} +\frac{1}{\lambda^{2r+b}}\right)\log^3\frac{2}{\delta},
$$
where 
$$
c_1=\max\left\{8\kappa^2,\,\, 8QZ\kappa,\,\, 382Q^2c_b,\,\, 72 R^2 c_{r\vee 1}^2\left(Q^2+C_{\kappa,R,D}^2\right)\,\, \right\} \cdot \max\left\{1,\,\log \frac{48 \kappa^2c_b}{\|\mathcal{L}_\infty\|}\right\}
$$

Therefore for the case $r\leq \frac{1}{2}$ it is enough to assume $\lambda= 2c_1 n^{-\frac{1}{2r+b}} \log^3\frac{2}{\delta}$ as long as $n\geq n_0$
where $n_0\geq n_0^{\frac{1}{2r+b}} \log(n_0)$ or equivalent $n_0\geq e^{\frac{2r+b}{2r+b-1}}$.  In case $r>\frac{1}{2}$ it remains to check if 
$n\geq \eta_5$. This holds if $n\geq  \frac{c_2}{\lambda^{1+b}} \log^3\frac{2}{\delta}$ with $c_2:=100\kappa^2c_b$ and therefore the condition on $n$ is fulfilled if 
$\lambda = C n^{-\frac{1}{2r+b}}\log^3\frac{2}{\delta}$  where $C:= \max\{2c_1,c_2\}$. The condition on $M$ is fulfilled if

\begin{align*}
M&\geq 
\begin{cases}
\frac{8 p\kappa^2 \beta_\infty}{\lambda}\vee C_{\delta,\kappa} & r\in\left(0,\frac{1}{2}\right)\\
\frac{8 p\kappa^2 \beta_\infty\vee C_1^{\frac{1}{r}}\vee C_2}{\lambda^{1+b(2r-1)} } \vee C_{\delta,\kappa} & r\in\left[\frac{1}{2},1\right] \\[5pt]
\frac{8 p\kappa^2 \beta_\infty\vee C_3}{\lambda^{2r}}\vee C_{\delta,\kappa} & r \in(1,\infty).\\
\end{cases}\\
 \end{align*}
Using $\lambda = C n^{-\frac{1}{2r+b}}\log^3\frac{2}{\delta}$ we have that the condition is fulfilled if

\begin{align*}
M\geq \frac{8 p\kappa^2 \beta_\infty\vee C_{\delta,\kappa} \vee \left(C_1^{\frac{1}{r}}\vee C_2\right)\mathbbm{1}_{r<1} \vee C_3}{C\log^3\frac{2}{\delta}}\cdot
\begin{cases}
n^{\frac{1}{2r+b}}& r\in\left(0,\frac{1}{2}\right)\\
n^{\frac{1+b(2r-1)}{2r+b}}  & r\in\left[\frac{1}{2},1\right] \\
n^{\frac{2r}{2r+b}} & r \in(1,\infty)\,.\\
\end{cases}\\
\end{align*}

Note that 
\begin{align*}
\frac{8 p\kappa^2 \beta_\infty\vee C_{\delta,\kappa} \vee \left(C_1^{\frac{1}{r}}\vee C_2\right)\mathbbm{1}_{\{r<1\}} \vee C_3}{C\log^3\frac{2}{\delta}} \leq \tilde{C} \log(n),
\end{align*}
for 
$\tilde{C}= 8p\kappa^2\cdot \max\left\{1,\,\log \frac{48 \kappa^2c_b}{\|\mathcal{L}_\infty\|}\right\}\cdot\frac{ 8\kappa^4\|\mathcal{L}_\infty\|^{-1}  \vee  8\kappa  \vee   16c_b\kappa^2\vee 4\kappa^4C_{\kappa,r}^2
}{C}$

Therefore the condition on $M$ holds true if 
\begin{align*}
M\geq \tilde{C}\log(n) \cdot \begin{cases}
n^{\frac{1}{2r+b}}& r\in\left(0,\frac{1}{2}\right)\\
n^{\frac{1+b(2r-1)}{2r+b}}  & r\in\left[\frac{1}{2},1\right] \\
n^{\frac{2r}{2r+b}} & r \in(1,\infty)\,.\\
\end{cases}\\
\end{align*}

Theorem \ref{theo2} now states:
\begin{align}
\|g_\rho-\mathcal{S}_M f_\lambda^M\|_{L^2(\rho_x)}&\leq \left(3 R c_{r\vee 1}+ 12D\left(\log\frac{2}{\delta}+ R c_{r\vee 1}\right)+ 12DR c_{\frac{1}{2}+r }  \right)\lambda^{r} \label{statement}\\
&= \left(3 R c_{r\vee 1}+ 12D\left(\log\frac{2}{\delta}+ R c_{r\vee 1}\right)+ 12DR c_{\frac{1}{2}+r }  \right)\left(C n^{-\frac{1}{2r+b}}\log^3\frac{2}{\delta}\right)^{r}\\
&\leq \bar{C} \log^{3r+1} \left(\frac{2}{\delta}\right) \, n^{-\frac{r}{2r+b}},
\end{align}
where 
$$
\bar{C}:=\left(3 R c_{r\vee 1}+ 12D\left(1+ R c_{r\vee 1}\right)+ 12DR c_{\frac{1}{2}+r }  \right)C^r, 
$$ 
provided that the events $E_1-E_9$ from \ref{events} occur. Since each event occurs with probability at least $1-\delta$ (see section \ref{III}), Proposition \ref{conditioning} proves that \eqref{statement} holds true with probability at least $1-9\delta$. Redefining  $\delta=9\delta$ proves the statement.

\end{proof}


\subsection{Appendix II}

\begin{proposition}
\label{conditioning}
Let  $E_i$ be events with probability at least $1-\delta_i$ and set
$$
E:=\bigcap^k_{i=1}E_i
$$
If we can show for some event $A$ that $\mathbb{P}(A|E)\geq 1-\delta_0$  then we also have
\begin{align*}
\mathbb{P}(A)&\geq\int_{E}\mathbb{P}(A|\omega)d\mathbb{P}(\omega)
\geq (1-\delta)\mathbb{P}(E)\\
&=(1-\delta)\left(1-\mathbb{P}\left(\bigcup_{i=1}^k (\Omega/E_i)\right)\right)\geq(1-\delta)\left(1-\sum_{i=1}^k\delta_i\right)>1-\sum_{i=0}^k\delta_i.
\end{align*}
\end{proposition}

\begin{proposition}[\cite{Muecke2017op.rates} (Proposition B.1.)]
\label{ineq1}
Let $B_{1}, B_{2}$ be two non-negative self-adjoint operators on some Hilbert space with $\left\|B_{j}\right\| \leq a, j=1,2$, for some non-negative a.
\begin{itemize}
\item[(i)] If $0 \leq r \leq 1$, then
$$
\left\|B_{1}^{r}-B_{2}^{r}\right\| \leq C_{r}\left\|B_{1}-B_{2}\right\|^{r},
$$
for some $C_{r}<\infty$.
\item[(ii)] If $r>1$, then
$$
\left\|B_{1}^{r}-B_{2}^{r}\right\| \leq C_{a, r}\left\|B_{1}-B_{2}\right\|,
$$
for some $C_{a, r}<\infty$. 
\end{itemize}
\end{proposition}

\begin{proposition}[Fujii et al., 1993, Cordes inequality]
\label{ineq2}
Let $A$ and $B$ be two positive bounded linear operators on a separable Hilbert space. Then
$$
\left\|A^s B^s\right\| \leq\|A B\|^s, \quad \text { when } 0 \leq s \leq 1 .
$$

\end{proposition}

\begin{proposition}[\cite{features} (Proposition 9)]
\label{ineq3}
Let $\mathcal{H}, \mathcal{K}$ be two separable Hilbert spaces and $X, A$ be bounded linear operators, with $A: \mathcal{H} \rightarrow \mathcal{K}$ and $B: \mathcal{H} \rightarrow \mathcal{H}$ be positive semidefinite.
$$
\left\|A B^\sigma\right\| \leq\|A\|^{1-\sigma}\|A B\|^\sigma, \quad \forall \sigma \in[0,1] .
$$

\end{proposition}

\begin{proposition}
\label{eq4}
Let $H_1, H_2$ be two separable Hilbert spaces and $\mathcal{S}: H_1 \rightarrow H_2$ a compact operator. Then for any function $f:[0,\|\mathcal{S}\|] \rightarrow[0, \infty[$,
$$
f\left(\mathcal{S} \mathcal{S}^*\right) \mathcal{S}=\mathcal{S} f\left(\mathcal{S}^* \mathcal{S}\right)
$$

\end{proposition}
\begin{proof}
The result can be proved using singular value decomposition of a compact operator.
\end{proof}

\begin{proposition}[\cite{spectral.rates} (Lemma 10)]
\label{ineqvolkan}
 Let $L$ be a compact, positive operator on a separable Hilbert space $H$ such that $\|L\| \leq \kappa^2$. Then for any $\lambda \geq 0$,
 \begin{align*}
\left\|(L+\lambda)^\alpha \phi_\lambda(L)\right\| &\leq 2 D \lambda^{-(1-\alpha)}, \quad \forall \alpha \in[0,1],\\
\left\|L^\alpha \phi_\lambda(L)\right\| &\leq  D \lambda^{-(1-\alpha)}, \quad \forall \alpha \in[0,1],
 \end{align*}

 where $D$ is defined in \eqref{def.phi}.
\end{proposition}

\begin{proposition}
\label{ineq5}
Assuming the event from Proposition \ref{OPbound2}:

$\begin{aligned}
&E_2=\left\{\left\|\mathcal{L}_{\infty,\lambda}^{-\frac{1}{2}}(\mathcal{L}_M-\mathcal{L}_\infty)\mathcal{L}_{\infty,\lambda}^{-\frac{1}{2}}\right\|\leq \frac{4 \kappa^2 \beta_\infty}{3M \lambda}+\sqrt{\frac{2 p\kappa^2 \beta_\infty}{M\lambda}}\right\}, &&\beta_\infty=\log \frac{4 \kappa^2(\mathcal{N}_{\mathcal{L}_\infty}(\lambda)+1)}{\delta\|\mathcal{L}_\infty\|}  ,
\end{aligned}$

holds true. Then we have for any $M\geq \frac{8 p\kappa^2 \beta_\infty}{\lambda}$,
\begin{align*}
\|f^*_\lambda\|_\infty &\leq  2 \kappa^{2r+1} R D \,\lambda^{-(\frac{1}{2}-r)^+},\\
\|f^*_\lambda\|_{\mathcal{H}_M} &\leq 2 \kappa^{2r} R D \,\lambda^{-(\frac{1}{2}-r)^+}.\\
\end{align*}

\end{proposition}

\begin{proof}
Note that $f_\lambda^*\in \mathcal{H}_M$. Therefore we obtain from the reproducing property and the definition  $g_\rho=\mathcal{L}^r_\infty h $,  for any $x\in\mathcal{X}$:
\begin{align*}
f_\lambda^*(x)&=\left\langle f_\lambda^*,K_M(x,.)\right\rangle_{\mathcal{H}_M} \\
&\leq \kappa \left\|f_\lambda^*\right\|_{\mathcal{H}_M} \\
&= \kappa \left\|\mathcal{S}^*_M\phi_\lambda(\mathcal{L}_M) g_\rho \right\|_{\mathcal{H}_M} \\
&=  \kappa \left\|\mathcal{L}_M^{\frac{1}{2}}\phi_\lambda(\mathcal{L}_M) \mathcal{L}_\infty^r h \right\|_{L^2(\rho_x)}.
\end{align*}
Using the assumption $\|h\|_{L^2(\rho_x)}\leq R$ we therefore have
\begin{align}
&\|f_\lambda^*\|_\infty
\leq  \kappa  R \left\|\mathcal{L}_M^{\frac{1}{2}}\phi_\lambda(\mathcal{L}_M) \mathcal{L}_{M,\lambda}^{(r\wedge\frac{1}{2})}\right\|  \left\|\mathcal{L}_{M,\lambda}^{-(r\wedge\frac{1}{2})}\mathcal{L}_{\infty}^{r}\right\|
= \kappa  R\,\, I\cdot II\label{fbound},\\
&\|f_\lambda^*\|_{\mathcal{H}_M}
\leq   R \left\|\mathcal{L}_M^{\frac{1}{2}}\phi_\lambda(\mathcal{L}_M) \mathcal{L}_{M,\lambda}^{(r\wedge\frac{1}{2})}\right\|  \left\|\mathcal{L}_{M,\lambda}^{-(r\wedge\frac{1}{2})}\mathcal{L}_{\infty}^{r}\right\|
=   R\,\, I\cdot II\label{fbound2}.
\end{align}
\begin{itemize}
\item [$I)$] For the first norm in \eqref{fbound} we have from Proposition \ref{ineqvolkan} that 
\begin{align*}
I&=\left\|\mathcal{L}_M^{\left(\frac{1}{2}+\left(r\wedge\frac{1}{2}\right)\right)}\phi_\lambda(\mathcal{L}_M) \right\|
\leq
\begin{cases}
D & r\geq \frac{1}{2}\\
D  \lambda^{r-\frac{1}{2}} & r<\frac{1}{2} 
\end{cases}\\
&\leq D \lambda^{-(\frac{1}{2}-r)^+} .
\end{align*}

\item [$II)$] For the second norm in \eqref{fbound} we have from the assumption and Proposition \ref{ineq2} that
\begin{align*}
II&= 
\begin{cases}
\left\|\mathcal{L}_{M,\lambda}^{-\frac{1}{2}}\mathcal{L}_{\infty,\lambda}^{r}\right\| \leq \left\|\mathcal{L}_{M,\lambda}^{-\frac{1}{2}}\mathcal{L}_{\infty,\lambda}^{\frac{1}{2}}\right\|\left\|\mathcal{L}_{\infty}^{r-\frac{1}{2}}\right\|\leq 2\kappa^{2r-1}& r\geq \frac{1}{2}\\[9pt]
\left\|\mathcal{L}_{M,\lambda}^{-r}\mathcal{L}_{\infty,\lambda}^{r}\right\| \leq \left\|\mathcal{L}_{M,\lambda}^{-\frac{1}{2}}\mathcal{L}_{\infty,\lambda}^{\frac{1}{2}}\right\|^{2r} \leq 4^r \leq 2 & r<\frac{1}{2} 
\end{cases}\\[5pt]
&\leq 2\kappa^{2r}
\end{align*}
\end{itemize}

Plugging the bounds of $I$ and $II$ into \eqref{fbound} leads to 
\begin{align*}
&\|f_\lambda^*\|_\infty\leq   2 \kappa^{2r+1} R D \,\lambda^{-(\frac{1}{2}-r)^+},\\
&\|f_\lambda^*\|_{\mathcal{H}_M}\leq   2 \kappa^{2r} R D \,\lambda^{-(\frac{1}{2}-r)^+}.
\end{align*}

\end{proof}

\begin{proposition}
\label{OPbound1}
Let $\mathcal{H}$ be a separable Hilbert space and let $A$ and $B$  be two bounded self-adjoint positive linear operators on $\mathcal{H}$ and $\lambda>0$. Then

$$
\left\|A_\lambda^{-\frac{1}{2}}B_\lambda ^{\frac{1}{2}}\right\| \leq(1-c)^{-\frac{1}{2}}, \quad \left\|A_\lambda ^{\frac{1}{2}}B_\lambda ^{-\frac{1}{2}}\right\| \leq (1+c)^{\frac{1}{2}}
$$
with
$$
c=\left\|B_\lambda ^{-\frac{1}{2}}(A-B)B_\lambda ^{-\frac{1}{2}}\right\|.
$$
\end{proposition}

\begin{proof}
The proof for the first inequality can for example be found in \cite{features} (Proposition 8). Using simple calculations the second inequality follows from 
\begin{align*}
\left\|(A+\lambda I)^{\frac{1}{2}}(B+\lambda I)^{-\frac{1}{2}}\right\|^2&=\left\|(B+\lambda I)^{-\frac{1}{2}}(A+\lambda I)(B+\lambda I)^{-\frac{1}{2}}\right\|\\
&\leq \left\|(B+\lambda I)^{-\frac{1}{2}}(A-B)(B+\lambda I)^{-\frac{1}{2}}\right\|+\|I\|\leq 1+c\\
\end{align*}

\end{proof}

\begin{proposition}[\cite{features} (Lemma 9)]
\label{OPbound3}
Assume that the event 

$\begin{aligned}
&E_6=\left\{\left\|\mathcal{L}_\infty-\mathcal{L}_M\right\|_{H S} \leq \left(\frac{2 \kappa^2}{M} + \frac{2 \kappa^2}{\sqrt{M}} \right)\log \frac{2}{\delta}\right\}\,.
\end{aligned}$

hold true then for any $M\geq8\kappa^4\|\mathcal{L}_\infty\|^{-1}\log^2 \frac{2}{\delta}$  we have
$$
\|\mathcal{L}_M\|\geq\frac{1}{2}\|\mathcal{L}_\infty\|.
$$
\end{proposition}
\begin{proof}
For $M\geq8\kappa^4\|\mathcal{L}_\infty\|^{-1}\log^2 \frac{2}{\delta}$ we obtain $\left\|\mathcal{L}_\infty-\mathcal{L}_M\right\|_{H S}\leq \frac{1}{2}\|\mathcal{L}_\infty\|$ 
and therefore
$$
\|\mathcal{L}_M\|  \geq \|\mathcal{L}_\infty\|-\left\|\mathcal{L}_\infty-\mathcal{L}_M\right\|_{H S}\geq \frac{1}{2} \|\mathcal{L}_\infty\|
$$
\end{proof}

\begin{proposition}
\label{OPbound5}
Providing Assumption \ref{ass:kernel} and assume the event 
$$
E_2 = \left\{\left\|\mathcal{L}_{\infty,\lambda}^{-\frac{1}{2}}(\mathcal{L}_M-\mathcal{L}_\infty)\mathcal{L}_{\infty,\lambda}^{-\frac{1}{2}}\right\|\leq \frac{4 \kappa^2 \beta_\infty}{3M \lambda}+\sqrt{\frac{2 p\kappa^2 \beta_\infty}{M\lambda}} \right\} ,
$$
where $\beta_\infty=\log \frac{4 \kappa^2(\mathcal{N}_{\mathcal{L}_\infty}(\lambda)+1)}{\delta\|\mathcal{L}_\infty\|} $ , holds true.
Then we have for any $M\geq \frac{8 p\kappa^2 \beta_\infty}{\lambda}$,
$$
\left\|\mathcal{L}_{M,\lambda}^{-\frac{1}{2}}\mathcal{L}_{\infty,\lambda}^{\frac{1}{2}}\right\| \leq 2, \quad \left\|\mathcal{L}_{M,\lambda}^{\frac{1}{2}}\mathcal{L}_{\infty,\lambda}^{-\frac{1}{2}}\right\| \leq 2.
$$

\end{proposition}

\begin{proof}
From the bound of event $E_2$ we have for any $\lambda>0$ ,
\begin{align}
\left\|\mathcal{L}_{\infty,\lambda}^{-\frac{1}{2}}(\mathcal{L}_M-\mathcal{L}_\infty)\mathcal{L}_{\infty,\lambda}^{-\frac{1}{2}}\right\|\leq \frac{4 \kappa^2 \beta_\infty}{3M \lambda}+\sqrt{\frac{2 p\kappa^2 \beta_\infty}{M\lambda}}.
\end{align}
From $M\geq \frac{8 p\kappa^2 \beta_\infty}{\lambda}$ we therefore obtain
\begin{align}
\left\|\mathcal{L}_{\infty,\lambda}^{-\frac{1}{2}}(\mathcal{L}_M-\mathcal{L}_\infty)\mathcal{L}_{\infty,\lambda}^{-\frac{1}{2}}\right\|\leq \frac{3}{4}
\end{align}

The result now follows from Proposition \ref{OPbound1}
\end{proof}

\begin{proposition}
\label{OPbound6}
Providing Assumption \ref{ass:kernel} and assume the events from Proposition \ref{OPbound2},

$\begin{aligned}
&E_1=\left\{\left\|\Sigma_{M,\lambda}^{-\frac{1}{2}}\left(\widehat{\Sigma}_{M}-\Sigma_{M}\right) \Sigma_{M,\lambda}^{-\frac{1}{2}}\right\|\leq\frac{4 \kappa^2 \beta_M}{3n \lambda}+\sqrt{\frac{2 \kappa^2 \beta_M}{n\lambda}}\right\},  &&\hspace{-0.5cm}\beta_M=\log \frac{4 \kappa^2(\mathcal{N}_{\mathcal{L}_M}(\lambda)+1)}{\delta\|\mathcal{L}_M\|},\\[7pt]
&E_2=\left\{\left\|\mathcal{L}_{\infty,\lambda}^{-\frac{1}{2}}(\mathcal{L}_M-\mathcal{L}_\infty)\mathcal{L}_{\infty,\lambda}^{-\frac{1}{2}}\right\|\leq \frac{4 \kappa^2 \beta_\infty}{3M \lambda}+\sqrt{\frac{2 p\kappa^2 \beta_\infty}{M\lambda}}\right\}, &&\hspace{-0.5cm}\beta_\infty=\log \frac{4 \kappa^2(\mathcal{N}_{\mathcal{L}_\infty}(\lambda)+1)}{\delta\|\mathcal{L}_\infty\|}  ,\\
&E_4=\left\{\left\|\mathcal{L}_{\infty,\lambda}^{-\frac{1}{2}}(\mathcal{L}_M-\mathcal{L}_\infty)\mathcal{L}_{\infty,\lambda}^{-\frac{1}{2}}\right\|_{HS}\leq  \left(\frac{4\kappa^2}{\lambda M}+\sqrt{\frac{ 4\kappa^2 \mathcal{N}_{\mathcal{L}_\infty}(\lambda) }{\lambda M}}\right)\log \frac{2}{\delta}\right\},\\
&E_6=\left\{\left\|\mathcal{L}_\infty-\mathcal{L}_M\right\|_{H S} \leq \left(\frac{2 \kappa^2}{M} + \frac{2 \kappa^2}{\sqrt{M}} \right)\log \frac{2}{\delta}\right\}\,.
\end{aligned}$

hold true. Then we have for  any $n\geq \frac{8\kappa^2 \tilde\beta}{\lambda}$ with $\tilde{\beta}:= \log \frac{4 \kappa^2(\left(1+2\log\frac{2}{\delta}\right)4\mathcal{N}_{\mathcal{L}_{\infty}}(\lambda)+1)}{\delta\|\mathcal{L}_\infty\|}$ and $M\geq \frac{8 p\kappa^2 \beta_\infty}{\lambda}\vee 8\kappa^4\|\mathcal{L}_\infty\|^{-1}\log^2 \frac{2}{\delta}$  that 
$$ 
\left\|\widehat{\Sigma}_{M,\lambda}^{-\frac{1}{2}}\Sigma_{M,\lambda}^{\frac{1}{2}}\right\| \leq 2, \quad  \left\|\widehat{\Sigma}_{M,\lambda}^{\frac{1}{2}}\Sigma_{M,\lambda}^{-\frac{1}{2}}\right\| \leq 2 .
$$

\end{proposition}

\begin{proof}
From the event $E_1$ we have for any $\lambda>0$ ,
\begin{align}
\left\|\Sigma_{M,\lambda}^{-\frac{1}{2}}\left(\widehat{\Sigma}_{M}-\Sigma_{M}\right) \Sigma_{M,\lambda}^{-\frac{1}{2}}\right\|&\leq\frac{4 \kappa^2 \beta}{3n \lambda}+\sqrt{\frac{2 \kappa^2 \beta}{n\lambda}},
\end{align}
with $\beta_M=\log \frac{4 \kappa^2(\mathcal{N}_{\mathcal{L}_M}(\lambda)+1)}{\delta\|\mathcal{L}_M\|}$ . 
Using the events $E_2,E_4$ together with $M\geq \frac{8 p\kappa^2 \beta_\infty}{\lambda}$ we obtain from Proposition \ref{prop:effecdim2} that 
\begin{align}
\mathcal{N}_{\mathcal{L}_{M}}(\lambda)\leq  \left(1+2\log\frac{2}{\delta}\right)4\mathcal{N}_{\mathcal{L}_{\infty}}(\lambda). \label{boundE24}
\end{align}
From the event $E_6$ we obtain from Proposition \ref{OPbound3} that

\begin{align}
\|\mathcal{L}_M\|\geq\frac{1}{2}\|\mathcal{L}_\infty\| \label{boundE56}
\end{align}

Note that the bounds of \eqref{boundE24} and \eqref{boundE56} imply $\beta_M\leq\tilde{\beta}= \log \frac{4 \kappa^2(\left(1+2\log\frac{2}{\delta}\right)4\mathcal{N}_{\mathcal{L}_{\infty}}(\lambda)+1)}{\delta\|\mathcal{L}_\infty\|}$ . Using this together with $n\geq \frac{8\kappa^2 \tilde\beta}{\lambda}$ we obtain

\begin{align}
\left\|\Sigma_{M,\lambda}^{-\frac{1}{2}}\left(\widehat{\Sigma}_{M}-\Sigma_{M}\right) \Sigma_{M,\lambda}^{-\frac{1}{2}}\right\|&\leq\frac{4 \kappa^2 \beta_M}{3n \lambda}+\sqrt{\frac{2 \kappa^2 \beta_M}{n\lambda}}\\
&\leq\frac{4 \kappa^2 \tilde\beta}{3n \lambda}+\sqrt{\frac{2 \kappa^2 \tilde\beta}{n\lambda}}\leq \frac{3}{4}
\end{align}
The result now follows from Proposition \ref{OPbound1}
\end{proof}

\begin{proposition}
\label{OPbound7}
Providing Assumption \ref{ass:kernel} and assume the events 

$\begin{aligned}
&E_2=\left\{\left\|\mathcal{L}_{\infty,\lambda}^{-\frac{1}{2}}(\mathcal{L}_M-\mathcal{L}_\infty)\mathcal{L}_{\infty,\lambda}^{-\frac{1}{2}}\right\|\leq \frac{4 \kappa^2 \beta_\infty}{3M \lambda}+\sqrt{\frac{2 p\kappa^2 \beta_\infty}{M\lambda}}\right\}, &&\beta_\infty=\log \frac{4 \kappa^2(\mathcal{N}_{\mathcal{L}_\infty}(\lambda)+1)}{\delta\|\mathcal{L}_\infty\|}  ,\\
&E_5=\left\{\left\|\mathcal{L}_{\infty,\lambda}^{-\frac{1}{2}}\left(\mathcal{L}_M-\mathcal{L}_\infty\right)\right\| \leq \left(\frac{2 
\kappa}{\sqrt{\lambda}M}+\sqrt{\frac{4 \kappa^2 \mathcal{N}_{\mathcal{L}_\infty}(\lambda) }{M}}\right)\log \frac{2}{\delta}\right\},\\
&E_6=\left\{\left\|\mathcal{L}_\infty-\mathcal{L}_M\right\|_{H S} \leq \left(\frac{2 \kappa^2}{M} + \frac{2 \kappa^2}{\sqrt{M}} \right)\log \frac{2}{\delta}\right\}\,.
\end{aligned}$

hold true.  Then for any  
\begin{align*}
M\geq 
\begin{cases}
\frac{8 p\kappa^2 \beta_\infty}{\lambda} & r\in\left(0,\frac{1}{2}\right)\\
\frac{(8 p\kappa^2 \beta_\infty)\vee C_1^{\frac{1}{r}}}{\lambda}\vee \frac{C_2}{\lambda^{1+b(2r-1)} } & r\in\left[\frac{1}{2},1\right] \\
\frac{C_3}{\lambda^{2r}} & r \in(1,\infty)
\end{cases}
\end{align*}
we have
$$
\left\|\mathcal{L}_{M,\lambda}^{-(r\vee1)}\mathcal{L}_{\infty,\lambda}^{r}\right\| \leq \frac{3}{\lambda^{(1-r)^+}},
$$
where  $C_1=2(4\kappa\log\frac{2}{\delta})^{2r-1}(8p\kappa^2\beta_\infty)^{1-r}$ ,  $C_2=4(4c_b\kappa^2\log\frac{2}{\delta})^{2r-1}(8p\kappa^2\beta_\infty)^{2-2r}$, \\ $C_3:= 4\kappa^4C_{\kappa,r}^2\log^2\frac{2}{\delta}$ and with $C_{\kappa,r}$ from Proposition \ref{ineq1}.
\end{proposition}

\begin{proof} For the proof we need to differ between the following three cases:
 \begin{itemize}

 \item CASE ($r\leq \frac{1}{2}$) :  From the event $E_2$ together with Proposition \ref{OPbound5} we have
\begin{align*}
\left\|\mathcal{L}_{M,\lambda}^{-(r\vee1)}\mathcal{L}_{\infty,\lambda}^{r}\right\|&=\left\|\mathcal{L}_{M,\lambda}^{-1}\mathcal{L}_{\infty,\lambda}^{r}\right\| \\
&\leq \lambda^{r-1}\left\|\mathcal{L}_{M,\lambda}^{-r}\mathcal{L}_{\infty,\lambda}^{r}\right\|\\
&\leq \lambda^{r-1}\left\|\mathcal{L}_{M,\lambda}^{-\frac{1}{2}}\mathcal{L}_{\infty,\lambda}^{\frac{1}{2}}\right\|^{2r} \leq 2^{2r}\lambda^{r-1}\leq 3\lambda^{r-1}.
\end{align*}

\item CASE ($r\in[\frac{1}{2},1]$) :  Using $\left\|\mathcal{L}_{\infty,\lambda}^{-1}\mathcal{L}_{\infty,\lambda}^{r}\right\|\leq \lambda^{r-1}$ we have

\begin{align}
\left\|\mathcal{L}_{M,\lambda}^{-(r\vee1)}\mathcal{L}_{\infty,\lambda}^{r}\right\|&=\left\|\mathcal{L}_{M,\lambda}^{-1}\mathcal{L}_{\infty,\lambda}^{r}\right\| \\
&\leq\left\|\left(\mathcal{L}_{M,\lambda}^{-1}-\mathcal{L}_{\infty,\lambda}^{-1}\right)\mathcal{L}_{\infty,\lambda}^{r}\right\|+\lambda^{r-1}.\label{OP8i}
\end{align}

For the norm of the last inequality we have from the algebraic identity 
\\$A^{-1}-B^{-1}=A^{-1}(A-B)B^{-1}$:
\begin{align*}
\left\|\left(\mathcal{L}_{M,\lambda}^{-1}-\mathcal{L}_{\infty,\lambda}^{-1}\right)\mathcal{L}_{\infty,\lambda}^{r}\right\|
=\left\|\mathcal{L}_{M,\lambda}^{-1}\left(\mathcal{L}_{M,\lambda}-\mathcal{L}_{\infty,\lambda}\right)\mathcal{L}_{\infty,\lambda}^{r-1}\right\|
\end{align*}
and from  event $E_2$ together with Proposition \ref{OPbound5} we further have
\begin{align*}
&\left\|\mathcal{L}_{M,\lambda}^{-1}\left(\mathcal{L}_{M,\lambda}-\mathcal{L}_{\infty,\lambda}\right)\mathcal{L}_{\infty,\lambda}^{r-1}\right\|\\
&\leq \lambda^{-\frac{1}{2}}\left\|\mathcal{L}_{M,\lambda}^{-\frac{1}{2}}\mathcal{L}_{\infty,\lambda}^{\frac{1}{2}}\right\| \left\|\mathcal{L}_{\infty,\lambda}^{-\frac{1}{2}}\left(\mathcal{L}_{M,\lambda}-\mathcal{L}_{\infty,\lambda}\right)\mathcal{L}_{\infty,\lambda}^{r-1}\right\|\\
&\leq 2\lambda^{-\frac{1}{2}} \left\|\mathcal{L}_{\infty,\lambda}^{-\frac{1}{2}}\left(\mathcal{L}_{M,\lambda}-\mathcal{L}_{\infty,\lambda}\right)\mathcal{L}_{\infty,\lambda}^{r-1}\right\|.\\
\end{align*}
Since $\sigma:=2-2r\leq 1$ we have from Proposition \ref{ineq3}  

\begin{align*}
& \left\|\mathcal{L}_{\infty,\lambda}^{-\frac{1}{2}}\left(\mathcal{L}_{M,\lambda}-\mathcal{L}_{\infty,\lambda}\right)\mathcal{L}_{\infty,\lambda}^{r-1}\right\|\\
 &\leq \left\|\mathcal{L}_{\infty,\lambda}^{-\frac{1}{2}}\left(\mathcal{L}_{M,\lambda}-\mathcal{L}_{\infty,\lambda}\right)\right\|^{2r-1}  \left\|\mathcal{L}_{\infty,\lambda}^{-\frac{1}{2}}\left(\mathcal{L}_{M,\lambda}-\mathcal{L}_{\infty,\lambda}\right)\mathcal{L}_{\infty,\lambda}^{-\frac{1}{2}}\right\|^{2-2r}.
\end{align*}
Using the bounds of the Events $E_2$ and $E_5$ we have for the last expression
\begin{align*}
\leq \left[\left(\frac{2 
\kappa}{\sqrt{\lambda}M}+\sqrt{\frac{4 \kappa^2 \mathcal{N}_{\mathcal{L}_\infty}(\lambda) }{M}}\right)\log \frac{2}{\delta}\right]^{2r-1}
\left(\frac{4 \kappa^2 \beta_\infty}{3M \lambda}+\sqrt{\frac{2 p\kappa^2 \beta_\infty}{M\lambda}}\right)^{2-2r}
\end{align*}
with $\beta_\infty=\log \frac{4 \kappa^2(\mathcal{N}_{\mathcal{L}_\infty}(\lambda)+1)}{\delta\|\mathcal{L}_\infty\|}$. Using this together with  $M\geq \frac{8 p\kappa^2 \beta_\infty}{\lambda} $ and the simple inequality $(a+b)^{2r-1}\leq a^{2r-1} + b^{2r-1}$ we have
\begin{align*}
&\left\|\left(\mathcal{L}_{M,\lambda}^{-1}-\mathcal{L}_{\infty,\lambda}^{-1}\right)\mathcal{L}_{\infty,\lambda}^{r}\right\|\\
&\leq 2\lambda^{-\frac{1}{2}}\left(\frac{4 
\kappa \log \frac{2}{\delta}}{\sqrt{\lambda}M}+\sqrt{\frac{4 \kappa^2 \mathcal{N}_{\mathcal{L}_\infty}(\lambda) \log \frac{2}{\delta}}{M}}\right)^{2r-1}
\left(\frac{4 \kappa^2 \beta_\infty}{3M \lambda}+\sqrt{\frac{2 p\kappa^2 \beta_\infty}{M\lambda}}\right)^{2-2r}\\
& \leq2\lambda^{-\frac{1}{2}} \left(\frac{4 
\kappa \log \frac{2}{\delta}}{\sqrt{\lambda}M}+\sqrt{\frac{4 \kappa^2 \mathcal{N}_{\mathcal{L}_\infty}(\lambda) \log \frac{2}{\delta}}{M}}\right)^{2r-1}
\left(2\sqrt{\frac{2p\kappa^2 \beta_\infty}{M\lambda}}\right)^{2-2r}\\
&\leq\frac{C_1}{\lambda M^r}+\sqrt{\frac{C_2'\mathcal{N}_{\mathcal{L}_\infty}(\lambda)^{2r-1}}{M\lambda^{3-2r}}}\leq\frac{C_1}{\lambda M^r}+\sqrt{\frac{C_2}{M\lambda^{3-2r+b(2r-1)}}},
\end{align*}
where we used in the last inequality the assumption $\mathcal{N}_{\mathcal{L}_\infty}(\lambda)\leq c_b \lambda^{-b}$  and set $C_1=2(4\kappa\log\frac{2}{\delta})^{2r-1}(8p\kappa^2\beta_\infty)^{1-r}$ ,  $C_2=4(4c_b\kappa^2\log^2\frac{2}{\delta})^{2r-1}(8p\kappa^2\beta_\infty)^{2-2r}$.
From $M\geq \frac{C_1^{\frac{1}{r}}}{\lambda} $ and $M\geq \frac{C_2}{\lambda^{1+b(2r-1)} }$ we obtain 
\begin{align*}
\left\|\left(\mathcal{L}_{M,\lambda}^{-1}-\mathcal{L}_{\infty,\lambda}^{-1}\right)\mathcal{L}_{\infty,\lambda}^{r}\right\|
\leq\frac{C_1}{\lambda M^r}+\sqrt{\frac{C_2}{M\lambda^{3-2r+b(2r-1)}}}
\leq 2\lambda^{r-1}.
\end{align*}
Plugging this bound into \eqref{OP8i} leads to
\begin{align*}
\left\|\mathcal{L}_{M,\lambda}^{-(r\vee1)}\mathcal{L}_{\infty,\lambda}^{r}\right\|\leq3 \lambda^{r-1}.
\end{align*}

\item CASE $(r\geq 1)$ :  
\begin{align*}
\left\|\mathcal{L}_{M,\lambda}^{-(r\vee1)}\mathcal{L}_{\infty,\lambda}^{r}\right\|&=\left\|\mathcal{L}_{M,\lambda}^{-r}\mathcal{L}_{\infty,\lambda}^{r}\right\|\\
&\leq 1+ \left\|\mathcal{L}_{M,\lambda}^{-r}\left(\mathcal{L}_{\infty,\lambda}^{r}-\mathcal{L}_{M,\lambda}^{r}\right)\right\|\\
&\leq 1+ \lambda^{-r}C_{\kappa, r}\left\|\mathcal{L}_{\infty,\lambda}-\mathcal{L}_{M,\lambda}\right\|,
\end{align*}
where $C_{\kappa, r}$ is defined in Proposition \ref{ineq1}.
From the bound of event $E_6$ we therefore obtain
\begin{align*}
&\left\|\mathcal{L}_{M,\lambda}^{-(r\vee1)}\mathcal{L}_{\infty,\lambda}^{r}\right\|\\
&\leq 1+ \lambda^{-r}C_{1, r}  \left(\frac{2 \kappa^2}{M} + \frac{2 \kappa^2}{\sqrt{M}} \right)\log \frac{2}{\delta}\leq 3
\end{align*}
where used $M\geq C_3 \lambda^{-2r}$, with $C_3:= 4\kappa^4C_{1,r}^2\log^2\frac{2}{\delta}$.
\end{itemize} 
\end{proof}

\begin{proposition}
\label{OPbound8}
Assume \ref{ass:source} with $r\geq\frac{1}{2}$ holds true and that the events 

$\begin{aligned}
&E_1=\left\{\left\|\Sigma_{M,\lambda}^{-\frac{1}{2}}\left(\widehat{\Sigma}_{M}-\Sigma_{M}\right) \Sigma_{M,\lambda}^{-\frac{1}{2}}\right\|\leq\frac{4 \kappa^2 \beta_M}{3n \lambda}+\sqrt{\frac{2 \kappa^2 \beta_M}{n\lambda}}\right\},  &&\hspace{-0.5cm}\beta_M=\log \frac{4 \kappa^2(\mathcal{N}_{\mathcal{L}_M}(\lambda)+1)}{\delta\|\mathcal{L}_M\|},\\[7pt]
&E_2=\left\{\left\|\mathcal{L}_{\infty,\lambda}^{-\frac{1}{2}}(\mathcal{L}_M-\mathcal{L}_\infty)\mathcal{L}_{\infty,\lambda}^{-\frac{1}{2}}\right\|\leq \frac{4 \kappa^2 \beta_\infty}{3M \lambda}+\sqrt{\frac{2 p\kappa^2 \beta_\infty}{M\lambda}}\right\}, &&\hspace{-0.5cm}\beta_\infty=\log \frac{4 \kappa^2(\mathcal{N}_{\mathcal{L}_\infty}(\lambda)+1)}{\delta\|\mathcal{L}_\infty\|}  ,\\
&E_3=\left\{\left\|\Sigma_{M,\lambda}^{-\frac{1}{2}}\left(\widehat{\Sigma}_{M}-\Sigma_{M}\right)\right\|_{HS}\leq  \left(\frac{2\kappa}{\sqrt{\lambda} n}+\sqrt{\frac{ 4\kappa^2 \mathcal{N}_{\mathcal{L}_M}(\lambda) }{ n}}\right)\log \frac{2}{\delta}\right\},\\
&E_4=\left\{\left\|\mathcal{L}_{\infty,\lambda}^{-\frac{1}{2}}(\mathcal{L}_M-\mathcal{L}_\infty)\mathcal{L}_{\infty,\lambda}^{-\frac{1}{2}}\right\|_{HS}\leq  \left(\frac{4\kappa^2}{\lambda M}+\sqrt{\frac{ 4\kappa^2 \mathcal{N}_{\mathcal{L}_\infty}(\lambda) }{\lambda M}}\right)\log \frac{2}{\delta}\right\},\\
&E_7=\left\{\left\|\widehat{\Sigma}_{M}-\Sigma_{M}\right\|_{HS} \leq \left(\frac{2 \kappa^2}{n} + \frac{2 \kappa^2}{\sqrt{n}} \right)\log \frac{2}{\delta}\right\}\,.
\end{aligned}$

hold true. Then we have for any $M\geq \frac{8 p\kappa^2 \beta_\infty}{\lambda}$ and  $n\geq \eta_1\vee\eta_5\vee \eta_6$ with 
$\eta_1=\frac{8\kappa^2 \tilde\beta}{\lambda}, \,\,\eta_5 = 100\kappa^2 \mathcal{N}_{\mathcal{L}_{\infty}}(\lambda)\lambda^{-1} \log^3 \frac{2}{\delta}$, $\eta_6= 8C_{\kappa,r}^2 \kappa^4\lambda^{-2r} \log^2\frac{2}{\delta}$ and $\tilde{\beta}:= \log \frac{4 \kappa^2(\left(1+2\log\frac{2}{\delta}\right)4\mathcal{N}_{\mathcal{L}_{\infty}}(\lambda)+1)}{\delta\|\mathcal{L}_\infty\|}$.

\begin{align*}
\left\|\widehat{\Sigma}_{M,\lambda}^{-(r\vee1)} \Sigma_{M,\lambda}^{(r\vee1)}\right\|\leq 2
\end{align*}
\end{proposition}

\begin{proof}
\item \textbf{Case $r\leq[\frac{1}{2},1]$:} From the bound of event $E_3$ we obtain
\begin{align*}
\left\|\widehat{\Sigma}_{M,\lambda}^{-(r\vee1)} \Sigma_{M,\lambda}^{(r\vee1)}\right\|&=\left\|\widehat{\Sigma}_{M,\lambda}^{-1} \Sigma_{M,\lambda}\right\|\\
&\leq\left\|\widehat{\Sigma}_{M,\lambda}^{-1}\left(\widehat{\Sigma}_{M}-\Sigma_{M}\right)\right\|_{HS}+1\\
&\leq\frac{1}{\sqrt{\lambda}} \left\|\widehat{\Sigma}_{M,\lambda}^{-\frac{1}{2}} \Sigma_{M,\lambda}^{\frac{1}{2}}\right\|\left\|\Sigma_{M,\lambda}^{-\frac{1}{2}}\left(\widehat{\Sigma}_{M}-\Sigma_{M}\right)\right\|_{HS}\\
&\leq \frac{1}{\sqrt{\lambda}} \left\|\widehat{\Sigma}_{M,\lambda}^{-\frac{1}{2}} \Sigma_{M,\lambda}^{\frac{1}{2}}\right\| \left(\frac{2\kappa}{\sqrt{\lambda} n}+\sqrt{\frac{ 4\kappa^2 \mathcal{N}_{\mathcal{L}_M}(\lambda) }{ n}}\right)\log \frac{2}{\delta}.
\end{align*}
Assuming the other events to hold true we have from Proposition \ref{OPbound6}
$$ 
\left\|\widehat{\Sigma}_{M,\lambda}^{-\frac{1}{2}}\Sigma_{M,\lambda}^{\frac{1}{2}}\right\| \leq 2
$$
and therefore
\begin{align}
\left\|\widehat{\Sigma}_{M,\lambda}^{-(r\vee1)} \Sigma_{M,\lambda}^{(r\vee1)}\right\|\leq \frac{2}{\sqrt{\lambda}} \left(\frac{2\kappa}{\sqrt{\lambda} n}+\sqrt{\frac{ 4\kappa^2 \mathcal{N}_{\mathcal{L}_M}(\lambda) }{ n}}\right)\log \frac{2}{\delta}.\label{ineqT2ii}
\end{align}

Assuming the events $E_2, E_4$ we have from  \ref{prop:effecdim2}:
$$
\mathcal{N}_{\mathcal{L}_{M}}(\lambda)\leq  \left(1+2\log\frac{2}{\delta}\right)4\mathcal{N}_{\mathcal{L}_{\infty}}(\lambda).
$$
Plugging this bound into \eqref{ineqT2ii} leads to
\begin{align}
\left\|\widehat{\Sigma}_{M,\lambda}^{-(r\vee1)} \Sigma_{M,\lambda}^{(r\vee1)}\right\|\leq \frac{2}{\sqrt{\lambda}} \left(\frac{2\kappa}{\sqrt{\lambda} n}+\sqrt{\frac{ 4\kappa^2 \left(1+2\log\frac{2}{\delta}\right)4\mathcal{N}_{\mathcal{L}_{\infty}}(\lambda) }{ n}}\right)\log \frac{2}{\delta}\leq 2
\end{align}

where we used $n\geq 100\kappa^2 \mathcal{N}_{\mathcal{L}_{\infty}}(\lambda)\lambda^{-1} \log^3 \frac{2}{\delta}$ in the last inequality.

\item \textbf{Case $r>1$:}   From Proposition \ref{ineq1} and the bound of event $E_7$ we have

\begin{align*}
\left\|\widehat{\Sigma}_{M,\lambda}^{-(r\vee1)} \Sigma_{M,\lambda}^{(r\vee1)}\right\|&=\left\|\widehat{\Sigma}_{M,\lambda}^{-r} \Sigma_{M,\lambda}^r\right\|\\
&\leq\lambda^{-r}\left\|\widehat{\Sigma}_{M}^r-\Sigma_{M}^r\right\|_{HS}+1\\
&\leq  \lambda^{-r}C_{\kappa,r}\left\|\widehat{\Sigma}_{M}-\Sigma_{M}\right\|_{HS}+1\\
&\leq  \lambda^{-r}C_{\kappa,r}\left(\frac{2 \kappa^2}{n} + \frac{2 \kappa^2}{\sqrt{n}} \right)\log \frac{2}{\delta}+1\leq2
\end{align*}
where we used $n\geq 8C_{\kappa,r}^2 \kappa^4\lambda^{-2r} \log^2\frac{2}{\delta}$ for the last inequality.

\end{proof}


\begin{proposition}
\label{prop:effecdim2}
Assume the events 

$\begin{aligned}
&E_2=\left\{\left\|\mathcal{L}_{\infty,\lambda}^{-\frac{1}{2}}(\mathcal{L}_M-\mathcal{L}_\infty)\mathcal{L}_{\infty,\lambda}^{-\frac{1}{2}}\right\|\leq \frac{4 \kappa^2 \beta_\infty}{3M \lambda}+\sqrt{\frac{2 p\kappa^2 \beta_\infty}{M\lambda}}\right\}, &&\hspace{-0.7cm}\beta_\infty=\log \frac{4 \kappa^2(\mathcal{N}_{\mathcal{L}_\infty}(\lambda)+1)}{\delta\|\mathcal{L}_\infty\|}  ,\\
&E_4=\left\{\left\|\mathcal{L}_{\infty,\lambda}^{-\frac{1}{2}}(\mathcal{L}_M-\mathcal{L}_\infty)\mathcal{L}_{\infty,\lambda}^{-\frac{1}{2}}\right\|_{HS}\leq  \left(\frac{4\kappa^2}{\lambda M}+\sqrt{\frac{ 4\kappa^2 \mathcal{N}_{\mathcal{L}_\infty}(\lambda) }{\lambda M}}\right)\log \frac{2}{\delta}\right\},
\end{aligned}$

hold true. Then we have for any $M\geq \frac{8 p\kappa^2 \beta_\infty}{\lambda}$,
$$
\mathcal{N}_{\mathcal{L}_{M}}(\lambda)\leq  \left(1+2\log\frac{2}{\delta}\right)4\mathcal{N}_{\mathcal{L}_{\infty}}(\lambda).
$$

\end{proposition}

\begin{proof}
\begin{align*}
\mathcal{N}_{\mathcal{L}_{M}}(\lambda)&\leq \text{ Tr}[\mathcal{L}_M\mathcal{L}_{\infty,\lambda}^{-1}]\left\|\mathcal{L}_{\infty,\lambda}^{\frac{1}{2}}\mathcal{L}_{M,\lambda}^{-\frac{1}{2}}\right\|^2\\
&=\left(\mathcal{N}_{\mathcal{L}_\infty}+\text{ Tr}[(\mathcal{L}_M-\mathcal{L}_\infty)\mathcal{L}_{\infty,\lambda}^{-1}]\right)\left\|\mathcal{L}_{\infty,\lambda}^{\frac{1}{2}}\mathcal{L}_{M,\lambda}^{-\frac{1}{2}}\right\|^2\\
&=\left(\mathcal{N}_{\mathcal{L}_\infty}+\|B\|_{HS}\right)\left\|\mathcal{L}_{\infty,\lambda}^{\frac{1}{2}}\mathcal{L}_{M,\lambda}^{-\frac{1}{2}}\right\|^2,
\end{align*}

where $B:=\mathcal{L}_{\infty,\lambda}^{-\frac{1}{2}}(\mathcal{L}_M-\mathcal{L}_\infty)\mathcal{L}_{\infty,\lambda}^{-\frac{1}{2}}$.  From Event $E_4$ we have
$$
\|B\|_{HS}\leq   2\left(\frac{2\kappa^2}{\lambda M}+\sqrt{\frac{ \kappa^2 \mathcal{N}_{\mathcal{L}_\infty}(\lambda) }{\lambda M}}\right)\log \frac{2}{\delta}.
$$
Using $\lambda> 4\kappa^2M^{-1}$ we obtain
$$
\|B\|_{HS}\leq 2\mathcal{N}_{\mathcal{L}_\infty}(\lambda) \log \frac{2}{\delta}
$$
Further we have from event $E_2$ and Proposition \ref{OPbound5}

$$
\left\|\mathcal{L}_{\infty,\lambda}^{\frac{1}{2}}\mathcal{L}_{M,\lambda}^{-\frac{1}{2}}\right\|^2\leq4.
$$

To sum up, we obtain 
\begin{align*}
\mathcal{N}_{\mathcal{L}_{M}}(\lambda)\leq\left(\mathcal{N}_{\mathcal{L}_\infty}+\|B\|_{HS}\right)\left\|\mathcal{L}_{\infty,\lambda}^{\frac{1}{2}}\mathcal{L}_{M,\lambda}^{-\frac{1}{2}}\right\|^2\leq \left(1+2\log\frac{2}{\delta}\right)4\mathcal{N}_{\mathcal{L}_{\infty}}(\lambda).
\end{align*}

\end{proof}

\subsection{Appendix III}
\label{III}

\begin{proposition}
\label{OPbound0}
Let $\mathcal{X}_1, \cdots, \mathcal{X}_m$ be a sequence of independently and identically distributed selfadjoint Hilbert-Schmidt operators on a separable Hilbert space. Assume that $\mathbb{E}\left[\mathcal{X}_1\right]=0$, and $\left\|\mathcal{X}_1\right\| \leq B$ almost surely for some $B>0$. Let $\mathcal{V}$ be a positive trace-class operator such that $\mathbb{E}\left[\mathcal{X}_1^2\right] \preccurlyeq \mathcal{V}$. Then with probability at least $1-\delta,(\delta \in] 0,1[)$, there holds
$$
\left\|\frac{1}{m} \sum_{i=1}^m \mathcal{X}_i\right\| \leq \frac{2 B \beta}{3 m}+\sqrt{\frac{2\|\mathcal{V}\| \beta}{m}}, \quad \beta=\log \frac{4 \operatorname{tr} \mathcal{V}}{\|\mathcal{V}\| \delta}
$$
\end{proposition}
\begin{proof}
The proposition was first established for matrices by  \cite{Tropp_2011}. For the general case including operators the proof can for example be found in  \cite{spectral.rates} (see Lemma 26).
\end{proof}

\begin{proposition}
\label{concentrationineq0}
The following concentration result for Hilbert space valued random variables can be found in (Caponnetto and De Vito, 2007 \cite{Caponetto}).\\
\\
Let $w_{1}, \cdots, w_{n}$ be i.i.d random variables in a separable Hilbert space with norm $\|.\|$. Suppose that there are two positive constants $B$ and $\sigma^{2}$ such that
\begin{align}
\mathbb{E}\left[\left\|w_{1}-\mathbb{E}\left[w_{1}\right]\right\|^{l}\right] \leq \frac{1}{2} l ! B^{l-2} V^{2}, \quad \forall l \geq 2 \label{cons}
\end{align}
Then for any $0<\delta<1 / 2$, the following holds with probability at least $1-\delta$,
$$
\left\|\frac{1}{n} \sum_{k=1}^{n} w_{n}-\mathbb{E}\left[w_{1}\right]\right\| \leq \left(\frac{2B}{n}+\frac{2V}{\sqrt{n}}\right) \log \frac{2}{\delta} .
$$
In particular, (\ref{cons}) holds if
$$
\left\|w_{1}\right\| \leq B / 2 \quad \text { a.s., } \quad \text { and } \quad \mathbb{E}\left[\left\|w_{1}\right\|^{2}\right] \leq V^{2} .
$$
\end{proposition}

\begin{proposition}
\label{OPbound2}
For any $\lambda>0$ define the following events,

$\begin{aligned}
&E_1=\left\{\left\|\Sigma_{M,\lambda}^{-\frac{1}{2}}\left(\widehat{\Sigma}_{M}-\Sigma_{M}\right) \Sigma_{M,\lambda}^{-\frac{1}{2}}\right\|\leq\frac{4 \kappa^2 \beta_M}{3n \lambda}+\sqrt{\frac{2 \kappa^2 \beta_M}{n\lambda}}\right\},  &&\beta_M=\log \frac{4 \kappa^2(\mathcal{N}_{\mathcal{L}_M}(\lambda)+1)}{\delta\|\mathcal{L}_M\|},\\[7pt]
&E_2=\left\{\left\|\mathcal{L}_{\infty,\lambda}^{-\frac{1}{2}}(\mathcal{L}_M-\mathcal{L}_\infty)\mathcal{L}_{\infty,\lambda}^{-\frac{1}{2}}\right\|\leq \frac{4 \kappa^2 \beta_\infty}{3M \lambda}+\sqrt{\frac{2 p\kappa^2 \beta_\infty}{M\lambda}}\right\}, &&\beta_\infty=\log \frac{4 \kappa^2(\mathcal{N}_{\mathcal{L}_\infty}(\lambda)+1)}{\delta\|\mathcal{L}_\infty\|}  ,\\
&E_3=\left\{\left\|\Sigma_{M,\lambda}^{-\frac{1}{2}}\left(\widehat{\Sigma}_{M}-\Sigma_{M}\right)\right\|_{HS}\leq  \left(\frac{2\kappa}{\sqrt{\lambda} n}+\sqrt{\frac{ 4\kappa^2 \mathcal{N}_{\mathcal{L}_M}(\lambda) }{ n}}\right)\log \frac{2}{\delta}\right\},\\
&E_4=\left\{\left\|\mathcal{L}_{\infty,\lambda}^{-\frac{1}{2}}(\mathcal{L}_M-\mathcal{L}_\infty)\mathcal{L}_{\infty,\lambda}^{-\frac{1}{2}}\right\|_{HS}\leq  \left(\frac{4\kappa^2}{\lambda M}+\sqrt{\frac{ 4\kappa^2 \mathcal{N}_{\mathcal{L}_\infty}(\lambda) }{\lambda M}}\right)\log \frac{2}{\delta}\right\},\\
&E_5=\left\{\left\|\mathcal{L}_{\infty,\lambda}^{-\frac{1}{2}}\left(\mathcal{L}_M-\mathcal{L}_\infty\right)\right\| \leq \left(\frac{2 
\kappa}{\sqrt{\lambda}M}+\sqrt{\frac{4 \kappa^2 \mathcal{N}_{\mathcal{L}_\infty}(\lambda) }{M}}\right)\log \frac{2}{\delta}\right\},\\
&E_6=\left\{\left\|\mathcal{L}_\infty-\mathcal{L}_M\right\|_{H S} \leq \left(\frac{2 \kappa^2}{M} + \frac{2 \kappa^2}{\sqrt{M}} \right)\log \frac{2}{\delta}\right\}\,,\\
&E_7=\left\{\left\|\widehat{\Sigma}_{M}-\Sigma_{M}\right\|_{HS} \leq \left(\frac{2 \kappa^2}{n} + \frac{2 \kappa^2}{\sqrt{n}} \right)\log \frac{2}{\delta}\right\}\,.
\end{aligned}$

Providing Assumption \ref{ass:input}  we have for any $\delta \in(0,1)$ that each of the above events holds true with probability at least $1-\delta$  .
\end{proposition}
\begin{proof}
The bound for $E_1$  follows exactly the same steps as in the proof of \cite{spectral.rates} (Lemma 18). The events $E_2-E_7$ have been bounded in \cite{features} ( see Proposition 6, Lemma 8 and Proposition 10). However, due to different assumptions and a different setting we attain slightly different bounds and therefore give the proof of  the events $E_2-E_7$ for completeness.\\

\textbf{$E_2)$} First note that $\mathcal{L}_M$ can be expressed by 
$$
\mathcal{L}_M=\frac{1}{M}\sum_{m=1}^M \sum_{i=1}^p \varphi_m^{(i)}\otimes\varphi_m^{(i)},
$$
where $\varphi_m(.)=\varphi(.,\omega_m).$ The above equality can be checked by simple calculations: 
\begin{align*}
\langle f, \mathcal{L}_Mg\rangle &= \int f(x) \int g(y)K_M(x,y) d\rho_x(y)  d\rho_x(x)\\
&= \int f(x) \frac{1}{M} \sum_{m=1}^M  \sum_{i=1}^p  \int  g(y)\varphi_m^{(i)}(y) \varphi_m^{(i)}(x)d\rho_x(y)  d\rho_x(x)\\
&= \int f(x) \frac{1}{M}\sum_{m=1}^M\sum_{i=1}^p  \left(\varphi_m^{(i)}\otimes\varphi_m^{(i)}\right)(g)(x) d\rho_x(x)\\
&= \left\langle f, \frac{1}{M}\sum_{m=1}^M \sum_{i=1}^p  \left(\varphi_m^{(i)}\otimes\varphi_m^{(i)}\right)g\right\rangle.
\end{align*}

Analog we have $\mathcal{L}_{\infty}=\mathbb{E}[ \sum_{i=1}^p \varphi^{(i)}\otimes\varphi^{(i)}]$.

Now define $\mathcal{X}_m:= \mathcal{L}_{\infty,\lambda}^{-\frac{1}{2}}(\mathcal{L}_{M}^{(m)} -\mathcal{L}_{\infty})\mathcal{L}_{\infty,\lambda}^{-\frac{1}{2}}$, with $\mathcal{L}_{M}^{(m)}:=\sum_{i=1}^p\varphi_m^{(i)}\otimes\varphi_m^{(i)}$.
We now obtain
\begin{align*}
\|\mathcal{X}_1\|\leq \left\|\mathcal{L}_{\infty,\lambda}^{-\frac{1}{2}}\mathcal{L}_{M}^{(m)} \mathcal{L}_{\infty,\lambda}^{-\frac{1}{2}}\right\| + \mathbb{E}\left\|\mathcal{L}_{\infty,\lambda}^{-\frac{1}{2}}\mathcal{L}_{M}^{(m)} \mathcal{L}_{\infty,\lambda}^{-\frac{1}{2}}\right\|\leq 2\frac{\kappa^2}{\lambda}:=B,
\end{align*}
where we used for the last inequality
$$
 \left\|\mathcal{L}_{\infty,\lambda}^{-\frac{1}{2}}\mathcal{L}_{M}^{(m)} \mathcal{L}_{\infty,\lambda}^{-\frac{1}{2}}\right\|\leq \lambda^{-1} \left\|\mathcal{L}_{M}^{(m)} \right\|\leq \frac{\kappa^2}{\lambda}.
$$
For the second moment we have from Jensen-inequality

\begin{align*}
\mathbb{E}\left[\mathcal{X}^2\right] &\preccurlyeq \mathbb{E}\left[\left(\mathcal{L}_{\infty,\lambda}^{-\frac{1}{2}}\mathcal{L}_{M}^{(m)} \mathcal{L}_{\infty,\lambda}^{-\frac{1}{2}}\right)^2\right]\\
&\preccurlyeq  \mathbb{E}\left[p\sum_{i=1}^p\left(\mathcal{L}_{\infty,\lambda}^{-\frac{1}{2}}\varphi_m^{(i)}\otimes\varphi_m^{(i)} \mathcal{L}_{\infty,\lambda}^{-\frac{1}{2}}\right)^2\right]\\
&=  \mathbb{E}\left[p\sum_{i=1}^p\left\|\mathcal{L}_{\infty,\lambda}^{-\frac{1}{2}}\varphi_m^{(i)}\right\|^2_{L^2_{\rho_x}}\mathcal{L}_{\infty,\lambda}^{-\frac{1}{2}}\varphi_m^{(i)}\otimes\varphi_m^{(i)} \mathcal{L}_{\infty,\lambda}^{-\frac{1}{2}}\right]\\
&\preccurlyeq \mathbb{E}\left[p\frac{\kappa^2}{\lambda}\mathcal{L}_{\infty,\lambda}^{-\frac{1}{2}}\mathcal{L}^{(m)}_{M}\mathcal{L}_{\infty,\lambda}^{-\frac{1}{2}}\right]\\
&= \frac{p\kappa^2}{\lambda}\mathcal{L}_{\infty}\mathcal{L}_{\infty,\lambda}^{-1}:=\mathcal{V}
\end{align*}

For $\beta=\log \frac{4 \operatorname{tr} \mathcal{V}}{\|\mathcal{V}\| \delta}$ we have
\begin{align*}
\beta&=\log \frac{4 \mathcal{N}_{\mathcal{L}_\infty}(\lambda)}{\|\mathcal{L}_{\infty}\mathcal{L}_{\infty,\lambda}^{-1}\| \delta}\\
&=\log \frac{4 \mathcal{N}_{\mathcal{L}_\infty}(\lambda)(\|\mathcal{L}_\infty\|+\lambda)}{\|\mathcal{L}_\infty\| \delta}\\[7pt]
&\leq \log \frac{4 \mathcal{N}_{\mathcal{L}_\infty}(\lambda)\|\mathcal{L}_\infty\|+4\operatorname{tr}\mathcal{L}_\infty}{\|\mathcal{L}_\infty\| \delta}\leq \log \frac{4\kappa^2 (\mathcal{N}_{\mathcal{L}_\infty}(\lambda)+1)}{\|\mathcal{L}_\infty\|\delta}.
\end{align*}

The claim now follows from Proposition \ref{OPbound0}.\\

\textbf{$E_3)$} Set $w_i:= \Sigma_{M,\lambda}^{-\frac{1}{2}}\xi_i$ with $\xi_i=K_{M,x_i}\otimes K_{M,x_i}$. Note that $\mathbb{E}[\xi_i]=\Sigma_M$ and 
\begin{align*}
\|w_i\|_{HS}&\leq \left\|\Sigma_{M,\lambda}^{-\frac{1}{2}}K_{M,x_i}\otimes K_{M,x_i}\right\|_{HS}\\
&\leq \lambda^{-1/2}\left\| K_{M,x_i} \right\|_{L^2_{\rho_x}}^2 \leq\frac{\kappa^2}{\sqrt{\lambda}}=:B
\end{align*}

For the second moment we have,

\begin{align*}
\mathbb{E}\left\|w_i^2\right\|_{HS}&\leq \kappa^2 \mathbb{E}\|\Sigma_{M,\lambda}^{-\frac{1}{2}}K_{M,x_i}\otimes K_{M,x_i}\Sigma_{M,\lambda}^{-\frac{1}{2}}\|_{HS} \\
&\leq \kappa^2\mathbb{E}\operatorname{tr}\left[\Sigma_{M,\lambda}^{-\frac{1}{2}}K_{M,x_i}\otimes K_{M,x_i}\Sigma_{M,\lambda}^{-\frac{1}{2}}\right]= \kappa^2\mathcal{N}_{\mathcal{L}_{M}}(\lambda)=:V^2
\end{align*}

The claim now follows from Proposition \ref{concentrationineq0}.

\textbf{$E_4)$}  Set $w_m:= \mathcal{L}_{\infty,\lambda}^{-\frac{1}{2}}(\mathcal{L}_{M}^{(m)} -\mathcal{L}_{\infty})\mathcal{L}_{\infty,\lambda}^{-\frac{1}{2}}$ . Note that we have
\begin{align*}
\|w_m\|_{HS}&\leq \left\|\mathcal{L}_{\infty,\lambda}^{-\frac{1}{2}}\mathcal{L}_{M}^{(m)} \mathcal{L}_{\infty,\lambda}^{-\frac{1}{2}}\right\|_{HS} + \operatorname{tr}\left[\mathcal{L}_\infty \mathcal{L}_{\infty,\lambda}^{-1}\right]\\
&\leq \left\|\mathcal{L}_{\infty,\lambda}^{-\frac{1}{2}} \left(\sum_{i=1}^p\varphi_m^{(i)}\otimes\varphi_m^{(i)}\right)\mathcal{L}_{\infty,\lambda}^{-\frac{1}{2}}\right\|_{HS} + \mathcal{N}_{\mathcal{L}_\infty}(\lambda)\\
&\leq \lambda^{-1}\sum_{i=1}^p\left\| \varphi_m^{(i)}\otimes\varphi_m^{(i)} \right\|_{HS} + \mathcal{N}_{\mathcal{L}_\infty}(\lambda)\\
&\leq \lambda^{-1}\sum_{i=1}^p\left\| \varphi_m^{(i)} \right\|_{L^2_{\rho_x}}^2 + \mathcal{N}_{\mathcal{L}_\infty}(\lambda)\leq\frac{2\kappa^2}{\lambda}=:B
\end{align*}

For the second moment we have,

\begin{align*}
\mathbb{E}\left\|w_m^2\right\|_{HS}\leq\mathbb{E}\operatorname{tr}\left[\left(\mathcal{L}_{\infty,\lambda}^{-\frac{1}{2}}\mathcal{L}_{M}^{(m)} \mathcal{L}_{\infty,\lambda}^{-\frac{1}{2}}\right)^2\right]\leq \frac{\kappa^2}{\lambda}\mathbb{E}\operatorname{tr}\left[\mathcal{L}_{\infty,\lambda}^{-\frac{1}{2}}\mathcal{L}_{M}^{(m)} \mathcal{L}_{\infty,\lambda}^{-\frac{1}{2}}\right]= \frac{\kappa^2}{\lambda}\mathcal{N}_{\mathcal{L}_{\infty}}(\lambda)=:V^2
\end{align*}
where we used $\|\mathcal{L}_{\infty,\lambda}^{-\frac{1}{2}}\mathcal{L}_{M}^{(m)} \mathcal{L}_{\infty,\lambda}^{-\frac{1}{2}}\|\leq \frac{\kappa^2}{\lambda}$ for the last inequality. 
The claim now follows from Proposition \ref{concentrationineq0}.

\textbf{$E_5)$} Set $w_m:= \mathcal{L}_{\infty,\lambda}^{-\frac{1}{2}}\mathcal{L}_{M}^{(m)}$ . Note that we have
\begin{align*}
\|w_m\|_{HS}&\leq \left\|\mathcal{L}_{\infty,\lambda}^{-\frac{1}{2}}\mathcal{L}_{M}^{(m)} \right\|_{HS} \\
&\leq \left\|\mathcal{L}_{\infty,\lambda}^{-\frac{1}{2}} \left(\sum_{i=1}^p\varphi_m^{(i)}\otimes\varphi_m^{(i)}\right)\right\|_{HS}\\
&\leq \lambda^{-1/2}\sum_{i=1}^p\left\| \varphi_m^{(i)} \right\|_{L^2_{\rho_x}}^2 \leq\frac{\kappa^2}{\sqrt{\lambda}}=:B
\end{align*}

For the second moment we have,

\begin{align*}
\mathbb{E}\left\|w_m^2\right\|_{HS}\leq \kappa^2 \mathbb{E}\|\mathcal{L}_{\infty,\lambda}^{-\frac{1}{2}}\mathcal{L}_{M}^{(m)} \mathcal{L}_{\infty,\lambda}^{-\frac{1}{2}}\|_{HS} \leq \kappa^2\mathbb{E}\operatorname{tr}\left[\mathcal{L}_{\infty,\lambda}^{-\frac{1}{2}}\mathcal{L}_{M}^{(m)} \mathcal{L}_{\infty,\lambda}^{-\frac{1}{2}}\right]= \kappa^2\mathcal{N}_{\mathcal{L}_{\infty}}(\lambda)=:V^2
\end{align*}

The claim now follows from Proposition \ref{concentrationineq0} together with the fact that the operator norm can be bounded by the Hilbert-Schmidt norm: $\|.\|\leq\|.\|_{HS}$ .

\textbf{$E_6)$} Set $w_m:= \mathcal{L}_{M}^{(m)}$ . Note that we have
\begin{align*}
\|w_m\|_{HS}&\leq \left\|\mathcal{L}_{M}^{(m)} \right\|_{HS} = \left\|\sum_{i=1}^p\varphi_m^{(i)}\otimes\varphi_m^{(i)}\right\|_{HS}\\
&\leq \sum_{i=1}^p\left\| \varphi_m^{(i)} \right\|_{L^2_{\rho_x}}^2 \leq \kappa^2=:B
\end{align*}

For the second moment we have,

\begin{align*}
\mathbb{E}\left\|w_m^2\right\|_{HS}\leq \kappa^4 =:V^2
\end{align*}

The claim now follows from Proposition \ref{concentrationineq0}

\textbf{$E_7)$}  Set $w_i:= \xi_i =K_{M,x_i}\otimes K_{M,x_i}$. Note that 
\begin{align*}
\|w_i\|_{HS}&= \left\|K_{M,x_i}\otimes K_{M,x_i}\right\|_{HS}\\
&\leq \left\| K_{M,x_i} \right\|_{L^2_{\rho_x}}^2 \leq \kappa^2 =:B
\end{align*}

For the second moment we have,

\begin{align*}
\mathbb{E}\left\|w_i^2\right\|_{HS}\leq \kappa^4=:V^2
\end{align*}

The claim now follows from Proposition \ref{concentrationineq0}.
\end{proof}

\begin{proposition}
\label{concentrationineq1}
Provided Assumptions \ref{ass:input} we have that the following event holds with probability at least $1-\delta$, 
\begin{align*}
E_8= \left\{\left\|\Sigma_{M,\lambda}^{-\frac{1}{2}}\widehat{\mathcal{S}}_{M}^{*}\left(y-\bar{g}_\rho\right)\right\|_{\mathcal{H}_M} \leq \left(\frac{4QZ\kappa}{\sqrt{\lambda}n}+\frac{4Q\sqrt{\mathcal{N}_{\mathcal{L}_M}(\lambda)}}{\sqrt{n}}\right) \log \frac{2}{\delta}\right\} \,.
\end{align*}

\end{proposition}

\begin{proof}
We want to use Proposition \ref{concentrationineq0} to prove the statement. Therefore define 
\\$w_i:= \left(y_i-g_\rho(x_i)\right)\Sigma_{M,\lambda}^{-\frac{1}{2}}K_{M,x_i}$.  Note that $\mathbb{E}w_i= 0 $ and $\frac{1}{n}\sum_{i=1}^n w_i = \Sigma_{M,\lambda}^{-\frac{1}{2}}\widehat{\mathcal{S}}_{M}^{*}\left(y-\bar{g}_\rho\right)$.
Further we have from Assumption \ref{ass:input},
\begin{align*}
&\mathbb{E}\left[\left\|w\right\|_{\mathcal{H}_M}^{l}\right] \\
&= \int_{\mathcal{X}} \int_{\mathcal{Y}}   \left(y-g_\rho(x)\right)^l \rho(dy|x)\|\Sigma_{M,\lambda}^{-\frac{1}{2}}K_{M,x}\|_{\mathcal{H}_M}^l\rho_x(dx)\\
&\leq2^{l-1} \int_{\mathcal{X}} \int_{\mathcal{Y}}   \left(|y|^l+Q^l\right)\rho(dy|x)\|\Sigma_{M,\lambda}^{-\frac{1}{2}}K_{M,x}\|_{\mathcal{H}_M}^l\rho_x(dx)\\
&\leq2^{l-1}\left(\frac{1}{2} l ! Z^{l-2} Q^2+Q^l\right) \int_{\mathcal{X}}\|\Sigma_{M,\lambda}^{-\frac{1}{2}}K_{M,x}\|_{\mathcal{H}_M}^l\rho_x(dx)\\
&\leq2^{l-1}\left(\frac{1}{2} l ! Z^{l-2} Q^2+Q^l\right) \sup_{x\in\mathcal{X}}\|\Sigma_{M,\lambda}^{-\frac{1}{2}}K_{M,x}\|_{\mathcal{H}_M}^{l-2}\int_{\mathcal{X}}tr\left(\Sigma_{M,\lambda}^{-1}K_{M,x}\otimes K_{M,x}\right)\rho_x(dx)\\
&\leq2^{l-1}\left(\frac{1}{2} l ! Z^{l-2} Q^2+Q^l\right) \left(\frac{\kappa}{\sqrt{\lambda}}\right)^{l-2}tr\left(\Sigma_{M,\lambda}^{-1}\int_{\mathcal{X}}K_{M,x}\otimes K_{M,x}\rho_x(dx)\right)\\
&\leq\frac{1}{2} l ! \left(\frac{2QZ\kappa}{\sqrt{\lambda}}\right)^{l-2}\left(2Q\sqrt{\mathcal{N}_{\mathcal{L}_M}(\lambda)}\right)^2\\
&= \frac{1}{2} l ! B^{l-2} V^{2}.
\end{align*}

Therefore the statement follows from Proposition \ref{concentrationineq0}.
\end{proof}

\begin{proposition}
\label{concentrationineq2}
Provided the assumption $\|g_\rho\|_\infty\leq Q$ and the bound of Proposition \ref{ineq5} : $\|f^*_\lambda\|_\infty \leq  C_{\kappa,R,D} \,\lambda^{-(\frac{1}{2}-r)^+}$,  where $C_{\kappa,R,D}=2 \kappa^{2r+1} R D$. Then the following event holds with probability at least $1-\delta$, 
\begin{align*}
E_9= \left\{\left|\frac{1}{n}\left\|\bar{g}_\rho-\widehat{\mathcal{S}}_M f_\lambda^*\right\|_2^2-\left\|g_\rho-\mathcal{S}_M f_\lambda^*\right\|_{L^2(\rho_x)}^2\right| \leq  2\left(\frac{B_\lambda}{n}+\frac{V_\lambda}{\sqrt{n}}\right) \log \frac{2}{\delta}\right\},
\end{align*}
where $B_\lambda:=4\left(Q^2+ C_{\kappa,R,D}^2 \,\lambda^{-2(\frac{1}{2}-r)^+}\right)$ and $V_\lambda:=\sqrt{2}\left(Q+C_{\kappa,R,D}  \,\lambda^{-(\frac{1}{2}-r)^+}\right)\left\|g_\rho-\mathcal{S}_M f_\lambda^*\right\|_{L^2(\rho_x)} $.
\end{proposition}

\begin{proof}

We want to use Proposition \ref{concentrationineq0} to prove the statement. Therefore define 
\\$w_i:= \left(g_\rho(x_i)- f_\lambda^*(x_i)\right)^2$.  Note that $\mathbb{E}w_1= \left\|g_\rho-\mathcal{S}_M f_\lambda^*\right\|_{L^2(\rho_x)}^2 $ and therefore
$$
\left|\frac{1}{n}\sum_{i=1}^n w_i-\mathbb{E}w_1\right|= \left|\frac{1}{n}\left\|\bar{g}_\rho-\widehat{\mathcal{S}}_M f_\lambda^*\right\|_2^2-\left\|g_\rho-\mathcal{S}_M f_\lambda^*\right\|_{L^2(\rho_x)}^2\right|
$$

It remains to bound $|w_i|$ and $\mathbb{E}w_1^2$. Using the assumption $\|g_\rho\|_\infty\leq Q$ and Proposition \ref{ineq5} we have
\begin{align*}
|w_i|\leq 2\left(Q^2+ C_{\kappa,R,D}^2 \,\lambda^{-2(\frac{1}{2}-r)^+}\right)
\end{align*}
and further
\begin{align*}
\mathbb{E}\left[w_1^2\right] &\leq   2\left(Q^2+ C_{\kappa,R,D}^2 \,\lambda^{-2(\frac{1}{2}-r)^+}\right)\mathbb{E}[w_1]\\&= 2\left(Q^2+C_{\kappa,R,D}^2  \,\lambda^{-2(\frac{1}{2}-r)^+}\right)\left\|g_\rho-\mathcal{S}_M f_\lambda^*\right\|_{L^2(\rho_x)}^2 
\end{align*}
Therefore the statement follows from Proposition \ref{concentrationineq0}.
\end{proof}


\end{document}